%% file: mpsn_icml21.tex
\documentclass{article}

\usepackage{microtype}
\usepackage{graphicx}
\usepackage{subfigure}
\usepackage{booktabs} 
\usepackage[usenames,dvipsnames]{xcolor}

\usepackage{hyperref}

\usepackage[accepted]{icml2021}
\usepackage{verbatim}
\usepackage{placeins}

\icmltitlerunning{Message Passing Simplicial Networks}

\input{math_commands.tex}

\usepackage{threeparttable}
\usepackage{hyperref}
\usepackage{url}
\usepackage{tabularx,ragged2e}
\usepackage{algorithmic}
\usepackage{enumitem}
\usepackage{multirow}
\usepackage{todonotes}
\usepackage{amsthm}

\newcommand{\rank}{\operatorname{rank}}

\newtheorem{lemma}[theorem]{Lemma}
\newtheorem{corollary}[theorem]{Corollary}
\newtheorem{remark}[theorem]{Remark}
\renewcommand{\eqref}[1]{(\ref{#1})} 
\expandafter\let\expandafter\oldproof\csname\string\proof\endcsname
\let\oldendproof\endproof
\renewenvironment{proof}[1][\proofname]{%
  \oldproof[\bf #1]%
}{\oldendproof}

\newcommand*{\ldblbrace}{\{\mskip-5mu\{}
\newcommand*{\rdblbrace}{\}\mskip-5mu\}}

\begin{document}

\twocolumn[
\icmltitle{Weisfeiler and Lehman Go Topological: Message Passing Simplicial Networks}



\icmlsetsymbol{equal}{*}

\begin{icmlauthorlist}
\icmlauthor{Cristian Bodnar}{equal,cambridge}
\icmlauthor{Fabrizio Frasca}{equal,twitter,imperial}
\icmlauthor{Yu Guang Wang}{equal,mpi,sjtu,unsw} \\
\icmlauthor{Nina Otter}{ucla} 
\icmlauthor{Guido Mont\'{u}far}{equal,mpi,ucla}
\icmlauthor{Pietro Li\`{o}}{cambridge}
\icmlauthor{Michael M. Bronstein}{twitter,imperial}
\end{icmlauthorlist}

%

\icmlaffiliation{unsw}{School of Mathematics and Statistics,
 University of New South Wales, Sydney, Australia}
\icmlaffiliation{ucla}{Department of Mathematics and Department of Statistics, University of California, Los Angeles, USA}
\icmlaffiliation{mpi}{Max Planck Institute for Mathematics in the Sciences, Leipzig, Germany}
\icmlaffiliation{sjtu}{Institute of Natural Sciences and School of Mathematical Sciences,
Shanghai Jiao Tong University, China}
\icmlaffiliation{cambridge}{Department of Computer Science and Technology, University of Cambridge, UK}
\icmlaffiliation{twitter}{Twitter, UK}
\icmlaffiliation{imperial}{Department of Computing, Imperial College London, UK}

\icmlcorrespondingauthor{Cristian Bodnar}{cb2015@cam.ac.uk}
\icmlcorrespondingauthor{Fabrizio Frasca}{ffrasca@twitter.com}
\icmlcorrespondingauthor{Yu Guang Wang}{yuguang.wang@mis.mpg.de}

\icmlkeywords{Simplicial Complexes, Message Passing, WL, Topology, Higher-Order Structures, Hodge Laplacian}

\vskip 0.3in
]



\printAffiliationsAndNotice{\icmlEqualContribution}


\begin{abstract}
The pairwise interaction paradigm of graph machine learning has predominantly governed the modelling of relational systems. However, graphs alone cannot capture the multi-level interactions present in many complex systems and the expressive power of such schemes was proven to be limited. To overcome these limitations, we propose Message Passing Simplicial Networks (MPSNs), a class of models that perform message passing on simplicial complexes (SCs). To theoretically analyse the expressivity of our model we introduce a Simplicial Weisfeiler-Lehman (SWL) colouring procedure for distinguishing non-isomorphic SCs. We relate the power of SWL to the problem of distinguishing non-isomorphic graphs and show that SWL and MPSNs are strictly more powerful than the WL test and not less powerful than the 3-WL test. We deepen the analysis by comparing our model with traditional graph neural networks (GNNs) with ReLU activations in terms of the number of linear regions of the functions they can represent. We empirically support our theoretical claims by showing that MPSNs can distinguish challenging strongly regular graphs for which GNNs fail and, when equipped with orientation equivariant layers, they can improve classification accuracy in oriented SCs compared to a GNN baseline. 

\end{abstract}

\vspace{-20pt}
\section{Introduction}\label{sec:intro}

Graphs are among the most common abstractions for complex systems of relations and interactions, arising in a broad range of fields from social science to high energy physics. %
Graph neural networks (GNNs), which trace their origins to the 1990s \cite{sperduti1994encoding, goller1996learning,gori2005new,scarselli2009graph,BrZaSzLe2013,li2015gated}, have recently achieved great success in learning tasks with graph-structured data. 

\begin{figure}[t]
    \centering
    \includegraphics[clip=true, trim=0cm 0cm 0cm .3cm,width=\columnwidth]{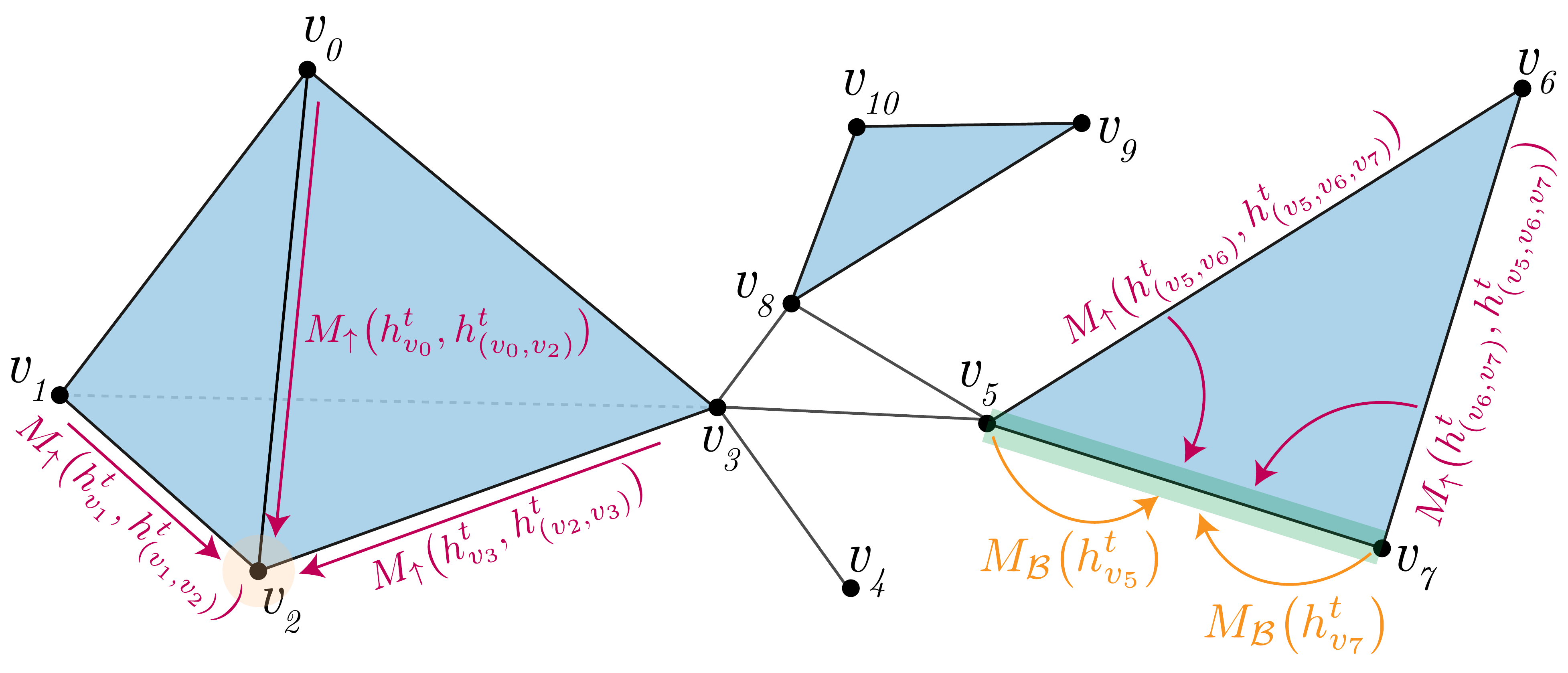} 
    \vspace{-17pt}
    \caption{Message Passing with upper and boundary adjacencies illustrated for vertex $v_2$ and edge $(v_5, v_7)$ in the complex.}
    \label{fig:sin}\vspace{-12pt}
\end{figure}

%
GNNs typically apply a local permutation-invariant  functions aggregating the neighbour features for each node, resulting in a permutation equivariant function on the graph \cite{maron2018invariant}. 
%
%
%
The design of the local aggregator is important: if chosen to be an injective function, the resulting GNN is equivalent in its expressive power to the Weisfeiler-Lehman (WL) graph  isomorphism test \cite{weisfeiler1968reduction,GIN,morris2019weisfeiler}. Due to their equivalence to the WL algorithm, message-passing type GNNs are unable to learn certain tasks on graphs. In particular, they are limited in their capability of detecting graph structure such as triangles or cliques \cite{chen2020can}. 
While other architectures equivalent to higher-dimensional $k$-WL tests have been proposed \cite{maron2019provably}, they suffer from high computational and memory complexity, and lack the key property of GNNs: Locality.


%
We tackle this problem by considering local higher-order interactions. Among many modelling frameworks that have been proposed to describe complex systems with higher-order relations \citep{battiston2020networks}, we specifically focus on simplicial complexes, a convenient middle ground between graphs (which are a particular case of simplicial complexes) and more general hypergraphs. Importantly, they offer strong mathematical connections to algebraic and differential topology and geometry. The simplicial Hodge Laplacian \citep{Barbarossa2020TopologicalSP, schaub2020random}, a discrete counterpart of the Laplacian operator in Hodge–de Rham theory \citep{rosenberg_1997}, provides a connection with the theory of spectral analysis and signal processing on these higher-dimensional  domains. 




We start by constructing a Simplicial Weisfeiler-Lehman (SWL) test for distinguishing non-isomorphic simplicial complexes. 
Motivated by this theoretical construction, we propose Message Passing Simplicial Networks (MPSNs), a message passing neural architecture for simplicial complexes that extends previous approaches such as GNNs and spectral simplicial convolutions \cite{bunch2020simplicial, ebli2020simplicial}. We then show that the proposed MPSN is as powerful as SWL. Strictly better than the conventional WL test \cite{weisfeiler1968reduction}, MPSN can be used to distinguish non-isomorphic graphs. We also show that the SWL test and MPSN are not less powerful than the 3-WL test \citep{cai1992optimal,morris2019weisfeiler}. 

Moreover, we explore the expressive power of GNNs and MPSNs  in terms of the number of linear regions of the functions they can represent \cite{Pascanu2013OnTN,montufar2014number}. 
We obtain bounds for the maximal number of linear regions of MPSNs and show a higher functional complexity than GNNs and simplicial  convolutional neural networks (SCNNs) \cite{ebli2020simplicial}, for which we provide optimal bounds that might be of independent interest. 
Proofs are presented in the Appendix. 

%


\vspace{-8pt}
\section{Background}

In this section, we focus on introducing the required background on (oriented) simplicial complexes. We assume basic familiarity with GNNs and the WL test.  

\begin{definition}[\citet{Nanda2021}]
Let $V$ be a non-empty vertex set. A simplicial complex $\gK$ is a collection of nonempty subsets of $V$ that contains all the singleton subsets of $V$ and is closed under the operation of taking subsets. 
\end{definition}

A member $\sigma = \{v_0, \ldots, v_k \} \in \gK$ with cardinality $k+1$ is called a \emph{$k$-dimensional simplex} or simply a $k$-simplex. Geometrically, one can see 0-simplices as \emph{vertices}, 1-simplices as \emph{edges}, 2-simplices as \emph{triangles}, and so on (see Figure \ref{fig:sin}). 

\begin{definition}[Boundary incidence relation]
\label{def:incidence_rel}
We say $\sigma \prec \tau$ iff $\sigma \subset \tau$ and there is no $\delta$ such that $\sigma \subset \delta \subset \tau$. 
\end{definition}

This relation describes what simplices are on the boundary of another simplex. For instance vertices $\{v_1\}, \{v_2\}$ are on the boundary of edge $\{ v_1, v_2 \}$ and edge $\{v_5, v_7 \}$ is on the boundary of triangle $\{v_5, v_6, v_7 \}$. 

In algebraic topology~\citep{Nanda2021} and discrete differential geometry~\citep{Crane:2013:DGP}, it is common to equip the simplices in a complex with an additional structure called an \emph{orientation}. An \emph{oriented $k$-simplex} is a $k$-simplex with a chosen order for its vertices. They can be uniquely represented as a tuple of vertices $(v_0, \ldots, v_k)$. A simplicial complex with a chosen orientation for all of its simplicies is called oriented. We note that mathematically, the choice of orientation is completely arbitrary. 

Intuitively, orientations describe a walk over the vertices of a simplex. For instance an edge $\{v_1, v_2 \}$ has two orientations $(v_1, v_2)$ and $(v_2, v_1)$, conveying a movement from $v_1$ to $v_2$ and from $v_2$ to $v_1$, respectively. Similarly, triangle $\{v_0, v_1, v_2 \}$ has six possible orientations given by all the permutations of its vertices. However, some of these, like $(v_0, v_1, v_2)$ and $(v_2, v_0, v_1)$, or $(v_0, v_2, v_1)$ and $(v_2, v_1, v_0)$, are equivalent, because (ignoring the starting point) they describe the same clockwise or counter-clockwise movement over the triangle. These two equivalence classes can be generalised to simplices of any dimension based on how the vertices are permuted (see \citet{hatcher_book} for details).   

\begin{definition}
\label{def:pos_neg_orient}
An oriented $k$-simplex is positively oriented if its  vertices form an even permutation and negatively oriented otherwise.  
\end{definition}

If $\gK$ is oriented, we can equip $\prec$ with this additional information. Consider two oriented simplices with $\tau \prec \sigma$ and $\mathrm{dim}(\sigma) > 1$. We say $\tau$ and $\sigma$ have the same orientation $\tau \prec_{+} \sigma$ if $\tau$ shows up in some even permutation of the vertices of $\sigma$. Otherwise, we have $\tau \prec_{-} \sigma$. Edges $(v_i, v_j)$ are a special case and we have $v_j \prec_{+} (v_i, v_j)$ and $v_i \prec_{-} (v_i, v_j)$. The oriented boundary relations $\prec_{+}, \prec_{-}$ can be encoded by the \emph{signed boundary matrices}. The $k$-th boundary matrix $B_k \in \sR^{S_{k-1} \times S_{k}}$ has entries
$B_k(i, j) = 1$ if $\tau_i \prec_{+} \sigma_j$, $-1$ if $\tau_i \prec_{-} \sigma_j$ and $0$, otherwise,
where $\mathrm{dim}(\sigma_j) = k$, $\mathrm{dim}(\tau_i) = k - 1$, and $S_{k}$ denotes the number of simplices of dimension $k$. Similarly, the boundary relation $\prec$ is encoded by $\vert B_k \vert$, the unsigned version of the matrix $B_k$. The $k$-th \emph{Hodge Laplacian} of the simplicial complex \citep{lim2015hodge, Barbarossa2020TopologicalSP, schaub2020random}, a diffusion operator for signals defined over the oriented $k$-simplices is defined as 
\begin{equation}
    L_k = B_k^\top B_k + B_{k+1} B_{k+1}^\top.
\end{equation}
We note that $L_0$ gives the well-known graph Laplacian.

\section{Simplicial WL Test}\label{S:SWL}
The deep theoretical link between the WL graph isomorphism test and message passing GNNs are well known \citep{GIN}. We exploit this connection to develop a simplicial version of the WL test with the ultimate goal of deriving a message passing procedure that can retain the expressive power of the test. We call this simplicial colouring algorithm \emph{Simplicial WL} (SWL). 
We outline 
below the steps of SWL in the most general sense.
\begin{enumerate}[leftmargin=*, topsep=0pt,itemsep=-0.5ex,partopsep=1ex,parsep=1ex]
    \item Given a complex $\mathcal{K}$, all the simplices $\sigma \in \mathcal{K}$ are initialised with the same colour. 
    \item Given the colour $c_\sigma^t$ of simplex $\sigma$ at iteration $t$, we compute the colour of simplex $\sigma$ at the next iteration $c_\sigma^{t+1}$, by perfectly hashing the multi-sets of colours belonging to the adjacent simplices of $\sigma$. 
    \item The algorithm stops once a stable colouring is reached. Two simplicial complexes are considered non-isomorphic if the colour histograms at any level of the complex are different. 
\end{enumerate}

A crucial choice has to be made about what simplices are considered to be adjacent in step two. The incidence relation from Definition \ref{def:incidence_rel} can be used to construct four types of (local) adjacencies. While all these adjacencies show up in graphs in some form, only one of them is typically used due to the lack of simplices of dimension above one. 

\begin{definition}
Consider a simplex $\sigma \in \gK$. Four types of adjacent simplices can be defined:
\begin{enumerate}[leftmargin=*, topsep=0pt,itemsep=-0.5ex,partopsep=1ex,parsep=1ex]
    \item Boundary adjacencies $\gB(\sigma) = \{ \tau \mid \tau \prec \sigma \}$.
    \item Co-boundary adjacencies $\gC(\sigma) = \{ \tau \mid \sigma \prec \tau \}$.
    \item Lower-adjacencies $\ndown(\sigma) = \{ \tau \mid \exists \delta,\ \delta \prec \tau \wedge \delta \prec \sigma \}$
    \item Upper-adjacencies $\nup(\sigma) = \{ \tau \mid \exists \delta,\ \tau \prec \delta \wedge \sigma \prec \delta \}$
\end{enumerate}
\end{definition}
To see how these adjacencies are already present in graphs, we provide a few examples. The boundary simplices of an edge are given by its vertices. The co-boundary simplices of a vertex are given by the edges they are part of. The lower-adjacent edges are given by the common line-graph adjacencies. Finally, upper adjacencies between vertices give the regular graph adjacencies. 

Let us use these adjacencies to precisely define the multi-sets of colours used in the update rule of SWL in step two. 

\begin{definition}
Let $c^t$ be the colouring of SWL for the simplices in $\gK$ at iteration $t$. We define the following multi-sets of colours, corresponding to each type of adjacency:
\begin{enumerate}[leftmargin=*, topsep=0pt,itemsep=-0.5ex,partopsep=1ex,parsep=1ex]
    \item $c_\gB^t(\sigma) = \ldblbrace c_\tau^t \mid \tau \in \gB(\sigma) \rdblbrace$.
    \item $c_\gC^t(\sigma) = \ldblbrace c_\tau^t \mid \tau \in \gC(\sigma) \rdblbrace$.
    \item $c_\da^t(\sigma) = \ldblbrace (c_\tau^t, c_{\sigma \cap \tau}^t) \mid \tau \in \ndown(\sigma) \rdblbrace$ .
    \item $c_\ua^t(\sigma) = \ldblbrace (c_\tau^t, c_{\sigma \cup \tau}^t) \mid \tau \in \nup(\sigma) \rdblbrace$.
\end{enumerate}
\end{definition}
Note that in the last two types of adjacencies, for adjacent $k$-simplices $\sigma$ and $\tau$, we also include the colour of the common $(k-1)$-simplex $\sigma \cap \tau$ on their boundaries and the $(k+1)$-simplex $\sigma \cup \tau$ they are both on the boundary of, respectively. 

Finally, we obtain the following update rule, which contains the complete set of adjacencies:
$$
c_\sigma^{t+1} = \mathrm{HASH}\{c_\sigma^t, c_\gB^t(\sigma), c_\gC^t(\sigma), c_\da^t(\sigma), c_\ua^t(\sigma)\}.
$$
Starting from this formulation, we will now show that certain adjacencies can be removed without sacrificing the expressive power of the SWL test in terms of simplicial complexes that can be distinguished.  




\begin{theorem}
\label{thm:sparse swl}
SWL with $\mathrm{HASH}\bigl(c_\sigma^t, c_{\gB}^t(\sigma), c_{\ua}^t(\sigma)\bigr)$ is as powerful as SWL with the generalised update rule $\mathrm{HASH}\bigl(c_\sigma^t, c_{\gB}^t(\sigma), c_{\gC}^t(\sigma), c_{\da}^t(\sigma), c_{\ua}^t(\sigma)\bigr)$. 
\end{theorem}

We note that other possible combinations of adjacencies might also fully preserve the expressive power of the general SWL test. Additionally, this result comes from a (theoretical) colour-refinement perspective and it does not imply that the pruned adjacencies cannot be  useful in practice. 

However, this particular choice presents two important properties that make it preferable over other potential ones. First, by not considering lower adjacencies, the test has a computational complexity that is linear in the total number of simplicies in the complex. These computational aspects are discussed in more detail in Section \ref{sec:mpsns}. Second, when the test is applied to $1$-simplicial complexes, i.e. graphs, and only vertex colours are considered, SWL corresponds to the WL test. This is due to the fact that vertices have no boundary simplices and upper adjacencies are the usual graph adjacencies between nodes. 

\begin{figure}[h]
    \centering
    \includegraphics[width=\columnwidth]{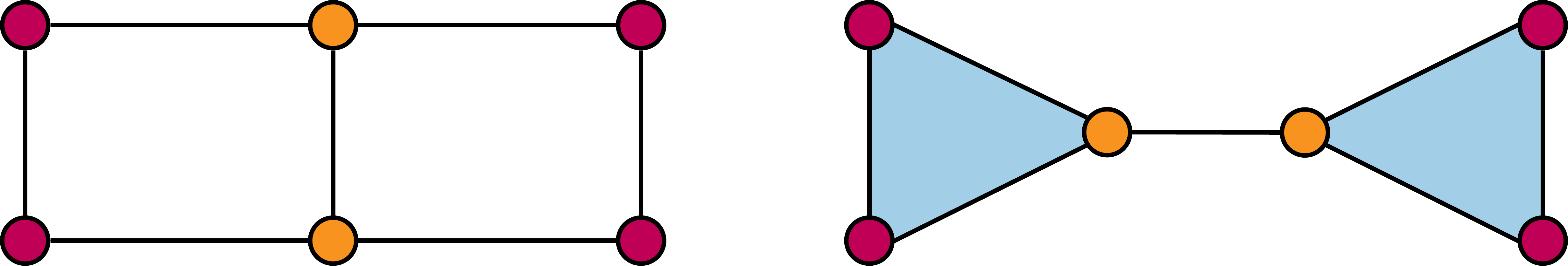}
    \vspace{-15pt}
    \caption{Two graphs that cannot be distinguished by 1-WL, but have distinct clique complexes (the second contains triangles).}
    \label{fig:cc_expresiveness}
    \vspace{-5pt}
\end{figure}

\paragraph{Clique Complexes and WL} The \emph{clique complex} of a graph $G$ is the simplicial complex $\gK$ with the property that if nodes $\{v_0, \ldots v_k\}$ form a clique in $G$, then  simplex $\{v_0, \ldots v_k\} \in \gK$. In other words, every $(k+1)$-clique in $G$ becomes a $k$-simplex in $\gK$. Evidently, this is an injective transformation from the space of non-isomorphic graphs to the space of non-isomorphic simplicial complexes. Therefore, the SWL procedure can be used to distinguish a pair of non-isomorphic graphs, by comparing their clique complexes. We call this a \emph{lifting transformation}. By taking this pre-processing step, we can link the expressive powers of WL and SWL: 

\begin{theorem}
\label{theo:swl_more_powerful_than_wl}
SWL with a clique complex lifting is strictly more powerful than WL. 
\end{theorem}

We present in Figure \ref{fig:cc_expresiveness} a pair of graphs that cannot be distinguished by the WL test, but whose clique complexes can be distinguished by SWL. 

By applying SWL to clique-complexes, even harder examples of non-isomorphic graph pairs can be told apart. One such example is the pair of graphs reported in Figure \ref{fig:SR}. This pair corresponds to the smallest pair of Strongly Regular non-isomorphic graphs with same parameters SR($16$,$6$,$2$,$2$). Our approach can distinguish the two graphs while the $3$-WL test fails, which is due to the fact that the two are associated with distinct clique complexes. This illustrates that our approach is no less powerful than the $3$-WL test. The following theorem confirms this observation.

\begin{theorem}\label{thm:swl_noless_than_3wl}
SWL is not less powerful than 3-WL. 
\end{theorem}

\begin{figure}
    \centering
    \includegraphics[clip=true,trim=0cm .45cm 0cm .45cm,width=0.75\columnwidth]{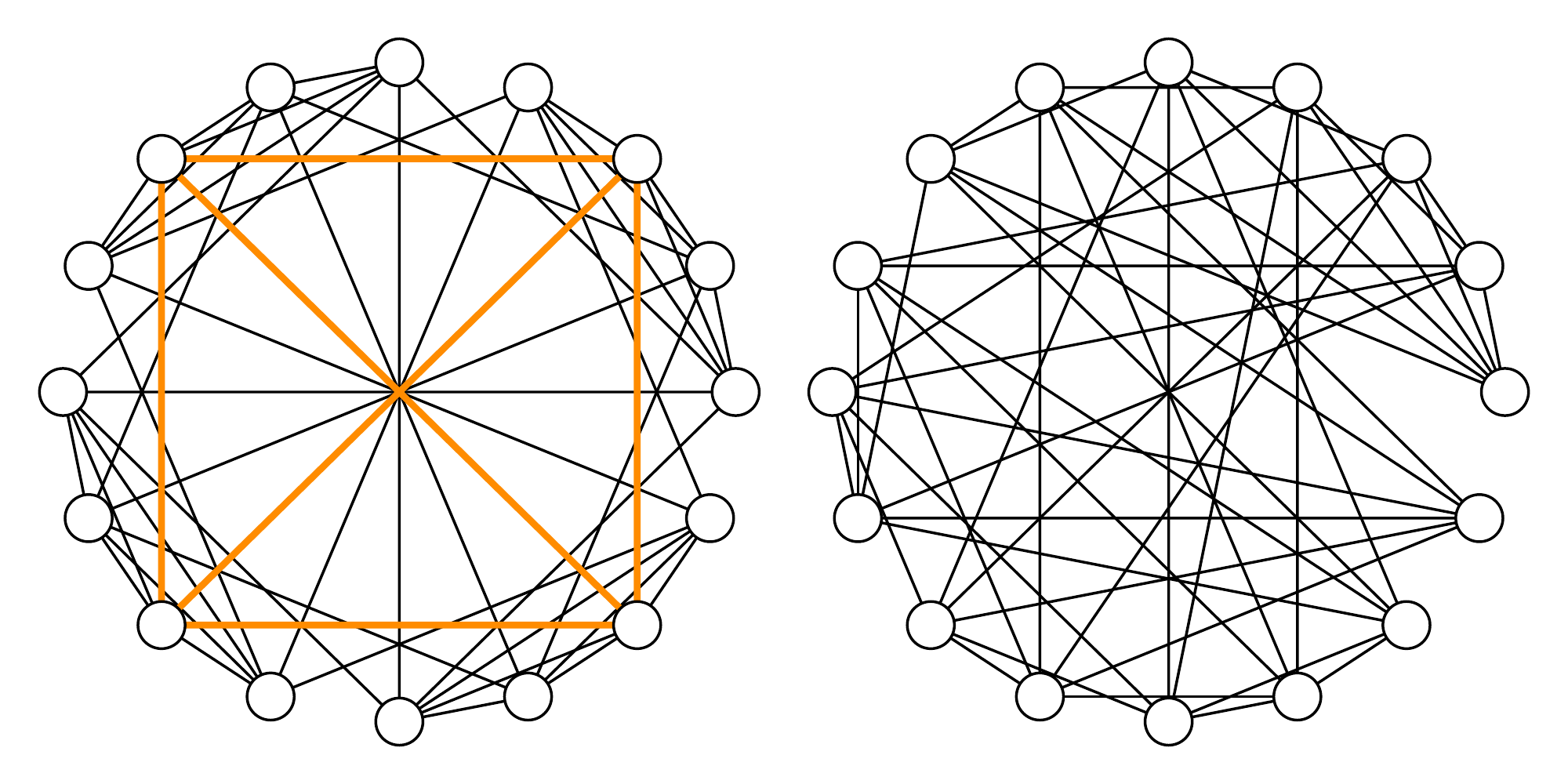} 
    \vspace{-10pt}
    \caption{Rook’s 4x4 graph and the Shrikhande graph: Strongly Regular non-isomorphic graphs with parameters SR(16,6,2,2). Our approach can distinguish them: only Rook's graph (left) possesses $4$-cliques (orange) and thus the two graphs are associated with distinct clique complexes. The 3-WL test fails to distinguish them.}
    \label{fig:SR}
    \vspace{-8pt}
\end{figure}

\section{Message Passing Simplicial Networks}\label{sec:mpsns}

\paragraph{MPSN} We propose a message passing model using the following message passing operations based on the four types of messages discussed in the previous section. For a simplex $\sigma$ in a complex $\gK$ we have:
\begin{align}
    m_{\gB}^{t+1}(\sigma) &= \text{AGG}_{\tau \in \gB(\sigma)}\Big(M_{\gB}\big(h_\sigma^{t}, h_\tau^{t}\big)\Big) \\
    m_{\gC}^{t+1}(\sigma) &= \text{AGG}_{\tau \in \gC(\sigma)}\Big(M_{\gC}\big(h_\sigma^{t}, h_\tau^{t}\big)\Big) \\
    m_{\downarrow}^{t+1}(\sigma) &= \text{AGG}_{\tau \in \gN_{\downarrow}(\sigma)}\Big(M_{\downarrow}\big(h_\sigma^{t}, h_\tau^{t}, h_{\sigma \cap \tau}^t\big)\Big) \\
    m_{\uparrow}^{t+1}(\sigma) &= \text{AGG}_{\tau \in \gN_{\uparrow}(\sigma)}\Big(M_{\uparrow}\big(h_\sigma^{t}, h_\tau^{t}, h_{\sigma \cup \tau}^t\big)\Big). 
\end{align}
Then, the update operation takes into account these four types of incoming messages and the previous colour of the simplex:
\begin{equation}\label{eq:mpsn}
    h_\sigma^{t+1} = U\Big(h_\sigma^{t}, m_{\gB}^{t}(\sigma), m_{\gC}^{t}(\sigma), m_{\downarrow}^{t+1}(\sigma), m_{\uparrow}^{t+1}(\sigma) \Big).
\end{equation}
To obtain a global embedding for a $p$-simplicial complex $\gK$ from an MPSN with $L$ layers, the readout function takes as input the multi-sets of features corresponding to all the dimensions of the complex:
\begin{equation}
    h_\gK = \text{READOUT}(\ldblbrace h_\sigma^L \rdblbrace_{\mathrm{dim}(\sigma)=0}, \ldots, \ldblbrace h_\sigma^L\rdblbrace_{\mathrm{dim}(\sigma)=p}). \nonumber
\end{equation}

\paragraph{Orientations in Message Passing} If the underlying complex also has a particular orientation, the message, aggregate, update and readout functions can be parametrised to use this information. More specifically, they can make use of the \emph{relative orientations} between adjacent simplices, which can be $\pm1$. The relative orientations of the neighbours of the $i$-th $k$-simplex are given by the non-zero entries in $B_k(\cdot, i)$ (boundary adjacencies), $B_{k+1}(i, \cdot)$ (co-boundary adjacencies), $B_k^\top B_k(i, \cdot)$ (lower adjacencies), and $B_{k+1} B_{k+1}^\top (i, \cdot)$ (upper adjacencies). For the last two adjacencies, the matrix multiplications encode how two adjacent simplices are oriented with respect to the lower- or higher-dimensional simplex they share, respectively. 





\paragraph{Expressive Power by WL} We now describe the expressive power of MPSNs in relation to SWL and WL. First, we need to analyse the ability of MPSNs to distinguish non-isomorphic simplicial complexes. 

\begin{lemma}
\label{lemma:mpsns_at_most_as_powerful_as_swl}
MPSNs are at most as powerful as SWL in distinguishing non-isomorphic simplicial complexes. 
\end{lemma}

One may wonder whether MPSN can achieve the power of SWL. The answer is affirmative: 

\begin{theorem}
\label{thm:mpsn_as_powerful_as_swl}
MPSNs with sufficient layers and injective neighbourhood aggregators are as powerful as SWL. 
\end{theorem}

This theorem, combined with Theorem~\ref{theo:swl_more_powerful_than_wl}, provides an important corollary showing that MPSNs are not only suitable for statistical tasks on higher-dimensional simplicial complexes, but could potentially improve over standard GNNs on graph machine learning tasks.   

\begin{corollary}
There exists an MPSN that is more powerful than WL at distinguishing non-isomorphic graphs when using a clique-complex lifting.  
\end{corollary}

In particular, based on Theorem \ref{thm:sparse swl}, it is sufficient for such an MPSN to use  boundary and upper adjacencies. This result relies on MPSN's ability of mapping multi-sets of colours injectively and the fact that neural networks that can learns such functions for multi-sets of bounded size exist \citep{GIN, Corso2020_PNA}. 

\paragraph{Relation to Spectral Convolutions} GNNs are also known for their relationship with spectral convolution operators on graphs obtained via graph Laplacian \citep{hammond2011wavelets}. Analogously to this, we show MPSNs generalise certain spectral convolutions on graphs derived from the higher-order simplicial Hodge Laplacian. The derivation and a more detailed discussion about simplicial spectral convolution are deferred to Appendix \ref{app:convolutions}. 

\begin{theorem}\label{thm:gen_workshop}
MPSNs generalise certain spectral convolution operators \citep{ebli2020simplicial, bunch2020simplicial} defined over simplicial complexes. 
\end{theorem}

\paragraph{Equivariance and Invariance} We now discuss how MPSNs handle the symmetries present in simplicial complexes. We offer here a high-level view of how these manifest and postpone a more rigorous treatment of both symmetries for Appendix \ref{app:equiv_inv}. 

Generalising GNNs, MPSN layers are equivariant with respect to relabelings of simplices in the complex.  

\begin{theorem}[Informal]
\label{theo:mpsn_equiv}
An MPSN layer is (simplex) permutation equivariant.
\end{theorem}

Additionally, if the complex is oriented, from a mathematical point of view, the choice of orientation is irrelevant. Therefore, MPSN layers should be equivariant with respect to changes in the orientation of the complex. Intuitively, this amounts to changes in the signature of the elements of the boundary matrices $B_k$ and rows of the feature matrices. 

\begin{theorem}[Informal]
\label{theo:mpsn_orient_equiv}
Consider an MPSN layer with a message function that multiplies the features of the neighbours by their relative orientation (i.e. $\pm1$). If the message, aggregate and update functions are odd (i.e. $f(x) = -f(-x)$), the layer is orientation equivariant.    
\end{theorem}

To construct an invariant model with respect to these transformations, it suffices to stack multiple equivariant layers like the ones above and to use an appropriate invariant readout function (more in Appendix \ref{app:equiv_inv}). 

\paragraph{Message Passing Complexity} A $d$-simplex $\sigma$ of a $p$-complex has $d+1$ boundary simplices and there are $\binom{d+1}{2}$ upper adjacencies between them. Then, a message passing procedure relying on Theorem \ref{thm:sparse swl}, which considers only these adjacencies, has a computational complexity $\Theta\big(\sum_{d=0}^p (d + 1) S_d + \binom{d+1}{2} S_d\big) = \Theta\big(\sum_{d=0}^p \binom{d+1}{2} S_d\big)$. If we consider $p$ to be a small constant, which is common for many real-world datasets, then the binomial coefficients can be absorbed in the bound, which results in a linear computational complexity in the size of the complex $\Theta(\sum_{d=0}^p S_d)$. Including co-boundaries does not increase this complexity, but lower adjacencies do. Because any $d$-complex can be down adjacent to any other $d$-complex, the worst case complexity is quadratic in the size of the complex $\mathcal{O}(\sum_{d=0}^p S_d^2)$.

\paragraph{Clique Complex Complexity} The number of $k$-cliques in a graph with $n$ nodes is upper-bounded by $\gO(n^k)$ and they can be listed in $\gO(a(G)^{k-2}m)$ time \citep{Chiba1985ArboricityAS}, where $a(G)$ is the arboricity of the graph (i.e. a measure of graph sparsity) and $m$ is the number of edges. Since the arboricity can be shown to be at most $\gO(m^{1/2})$ and $m \leq n^2$, all $k$-cliques can be listed in $\gO(n^{k-2}m)$. In particular, all triangles can be found in $\mathcal{O}(m^{3/2})$. For certain classes, such as planar graphs, where $a(G) \leq 3$, the complexity becomes linear in the size of the graph. 
%
%
Overall, the fact that the algorithm takes advantage of the sparsity of the graph makes its complexity strictly better than the $\Omega(n^k)$ of all $k$-GNNs~\citep{morris2019weisfeiler, morris2020weisfeiler, maron2019provably}.

\section{
Simplicial Networks by Linear Regions}

While the WL test has been used for studying the expressive power of GNNs, other tools have been used to study the expressive power of conventional neural networks, like fully connected and convolutional. One such tool is based on the number of linear regions of networks using piece-wise linear activations. This number has been used to draw distinctions between the expressive power of shallow and deep network architectures \cite{Pascanu2013OnTN,montufar2014number}. It can also be related to the approximation properties of the networks \cite{pmlr-v49-telgarsky16} and it has also been considered to shed light into the representational power of (standard) convolutional networks \cite{pmlr-v119-xiong20a}. 
We show how this tool can also be used to approximate the number of linear regions of the functions represented by graph, simplicial, and message passing simplicial networks. 
We focus on the case where the message function is a linear layer and AGG is sum followed by ReLU. 
We obtain new results in all cases, showing superior capacity of MPSNs. 
The details of the proofs of Theorems~\ref{thm:linear regions of GNNs}, \ref{thm:linear regions of SCNN}, and \ref{thm:linear regions of MPSN} are relegated to Appendix~\ref{app:linear regions}. 

\paragraph{GNN}
We start with the simple case of Graph Neural Networks (GNNs). 
A graph $G=(V,E,\omega)$ is a set of triplets with vertices $V=\{v_i\}_{i=1}^{S_0}$, edges $E\subseteq V\times V$, and edge weight function $\omega\colon 
E \to \R$. The graph has an adjacency matrix $A$ with the $(i,j)$th entry $a_{ij}=\omega(v_i,v_j)$. 
Each node has a $d$-dimensional feature, and we collect the feature vectors into a matrix $H^{\rm in}\in\R^{S_0\times d}$. 
We consider a GNN convolutional layer of the form
\begin{equation}\label{eq:gnn conv}
    H^{\rm out} = \psi\left(\mathcal{H}(A,H^{\rm in})W_0\right),
\end{equation}
where 
$\psi$ is the entry-wise ReLU, $\mathcal{H}(A,H)$ is an aggregation mapping, 
and $W_0\in \R^{d\times m}$ are the trainable weights. 
%

\begin{theorem}[Number of linear regions of a GNN layer]\label{thm:linear regions of GNNs} 
Consider a graph $G$ with 
$S_0$ nodes, node input features of dimension $d\geq1$, and node output features of dimension $m$. 
Suppose the aggregation function $\mathcal{H}$ as function of $H$ is linear and invertible. 
Then, the number of linear regions of the functions represented by a ReLU GNN layer \eqref{eq:gnn conv} has the optimal upper bound 
\begin{equation}\label{eq:upper bound GNN}
R_{\rm GNN} = \left(2\sum_{i=0}^{d-1}{m -1 \choose i}\right)^{S_0}.
\end{equation}
This applies to aggregation functions with no trainable parameters including GCN convolution \cite{KiWe2017}, spectral GNN \cite{defferrard2016convolutional,BrZaSzLe2013}, and traditional message passing \cite{Gilmer_etal2017}. 
\end{theorem}
The above result should be compared with the optimal upper bound for a standard dense ReLU layer without biases, 
which for $d$ inputs and $m$ outputs is $2\sum_{i=0}^{d-1}{m-1\choose i}$. 

The invertibility condition for the aggregation function $\mathcal{H}$ can be relaxed, but is satisfied by many commonly used graph convolutions:
i) For an undirected graph, the normalised adjacency matrix has non-negative eigenvalues. If the eigenvalues are all positive, the aggregation function is invertible. 
ii) The Fourier transform is the square matrix of eigenvectors, as used in the spectral GNN \cite{BrZaSzLe2013}. When the graph Laplacian is non-singular, the Fourier transform matrix is invertible. 
iii) For the transform $\Phi$ by graph wavelet basis, Haar wavelet basis or graph framelets, $\Phi$ is invertible in all cases \cite{GWNN,LiHANet2019,zheng2020graph,zheng2020decimated,zheng2021framelets,wang2020haar}. So the bound in Theorem~\ref{thm:linear regions of GNNs} applies to them.

\paragraph{SCNN}
Simplicial Complex Neural Networks (SCNNs) were proposed by \citet{ebli2020simplicial}. We consider a version of their model using only the first power of a Laplacian matrix, generically denoted here by $M_n$:
\begin{equation}\label{eq:scnn layer}
   H^{\rm out}_n = \psi 
  \bigl(M_n H^{\rm in}_n W_n\bigr),\quad n=0,\dots,p.
\end{equation}
In this type of layer, the features on simplices of different dimensions $n=0,1\ldots,p$ are computed in parallel. 

\begin{theorem}[Number of linear regions for an SCNN layer]\label{thm:linear regions of SCNN} 
Consider a $p$-dimensional simplicial complex with $S_n$ $n$-simplicies for $n=0,1,\ldots, p$. 
Suppose that each 
$M_n$ is invertible. 
Then the number of linear regions of the functions represented by a ReLU SCNN layer \eqref{eq:scnn layer} 
has the optimal upper bound 
\begin{equation}\label{eq:upper bound SCNN}
R_{\rm SCNN} =
 \prod_{n=0}^p\left(2\sum_{i=0}^{d_n-1}{m_n-1\choose i}\right)^{S_n}, 
\end{equation}
where, for each of the $n$-simplices, $d_n$ is the input feature dimension and $m_n$ is the number of output features. 
\end{theorem}
The product over $n$ in \eqref{eq:upper bound SCNN} reflects the fact that the features over simplices of different dimensions do not interact. 
The GNN bound in Theorem~\ref{thm:linear regions of GNNs} is recovered as the special case of the SCNN bound 
with $p=0$.

It is instructive to compare the SCNN bound in Theorem~\ref{thm:linear regions of SCNN} with the optimal bound for a dense layer. 
By Roth's lemma \cite{roth1934}, $\operatorname{vec}(M_n H_n^{\rm in} W_n) = (W_n^\top\otimes M_n)\cdot \operatorname{vec}(H_n^{\rm in})$, where $\operatorname{vec}$ denotes column-by-column vectorization and $\otimes$ the Kronecker product. 
Hence, for each $n$, we can regard the SCNN layer as a standard layer $\psi(U x)$ with weight matrix $U = (W_n^\top\otimes M_n)\in\mathbb{R}^{(m_n S_n) \times (S_n d_n)}$ and input vector $x = \operatorname{vec}(H_n^{\rm in})\in\mathbb{R}^{S_n d_n}$. 
Notice that for the SCNN layer, the weight matrix has a specific structure. 
%
A standard dense layer with $S_n d_n$ inputs and $m_n S_n$ ReLUs with generic weights and no biases computes functions with $2\sum_{i=0}^{S_n d_n - 1} {m_n S_n-1\choose i}$ regions. 

\paragraph{MPSN} 
In our Message Passing Simplicial Network (MPSN), the features on simplices of different dimensions are allowed to interact. 
For a $p$-dimensional 
complex, denote the set of $n$-simplices by $\mathcal{S}_n$ with $S_n = |\mathcal{S}_n|$. Denote the $n$-simplex input feature dimension by $d_n$, and the output feature dimension by $m_n=m$, $n=0,\ldots, p$. 
We consider an MPSN layer with linear message functions, sum aggregation for all messages and an update function taking the sum of the messages followed by a ReLU activation. 
For each dimension $n$, the output feature matrix $H_{n}^{\rm out}$ equals:
\begin{equation}\label{eq:m}
    \psi\Bigl( M_{n} H_{n}^{\rm in} W_n 
    + U_{n} H_{n-1}^{\rm in} W_{n-1}
    + O_{n} H_{n+1}^{\rm in}W_{n+1}\Bigr),
\end{equation}
where $\psi$ is an entry-wise activation ($s\mapsto\max\{0,s\}$ for ReLU), 
$W_n\in \mathbb{R}^{d_n\times m_n}$ are trainable weight matrices and 
%
$M_n\in\mathbb{R}^{S_n\times S_n}$,
$U_n\in\mathbb{R}^{S_n\times S_{n-1}}$, and $O_{n}\in\mathbb{R}^{S_n\times S_{n+1}}$ are some choice of adjacency matrices for the simplicial complex. These could be the Hodge Laplacian matrix $L_n$ and the corresponding boundary matrices $B_n^\top$, $B_{n+1}$, or one of their variants (e.g. normalised). 

It is convenient to write the entire layer output in standard form. 
Using Roth's lemma and concatenating over $n$ we can write \eqref{eq:m} as (details in Appendix~\ref{app:linear regions})
\begin{equation}
H^{\rm out} = \psi (W H^{\rm in} ), 
\label{eq:MPSN full standard}
\end{equation}
where $H^{\rm in} = \operatorname{vec}([H_0^{\rm in}|H_1^{\rm in}|\cdots|H_p^{\rm in}]) \in \mathbb{R}^N$, 
$N=\sum_{n=0}^pS_n d_n$, 
$H^{\rm out} = \operatorname{vec}([H_0^{\rm out} |H_1^{\rm out} |\cdots|H_p^{\rm out}])\in\mathbb{R}^M$, $M=\sum_{n=0}^p S_n m$, 
and 
\begin{align}
W =& 
\left[\begin{smallmatrix}
W_0^\top \otimes M_0  & W^\top_{1} \otimes O_0 & & & \\
W^\top_{0} \otimes U_{1} & W^\top_1\otimes M_1 & W^\top_2\otimes O_1 & \\
& W^\top_{1} \otimes U_{2} & W^\top_2\otimes M_2 & W^\top_3\otimes O_2  \\
&&&\ddots &
\end{smallmatrix}\right]. 
\label{eq:MPSNdimkron}
\end{align}
%
We study the number of linear regions of the function \eqref{eq:MPSN full standard} with ReLU based on the matrix $W \in \mathbb{R}^{M\times N}$. 
For each of the output coordinates $i\in\{1,\ldots, M\}$, the ReLU splits the input space $\mathbb{R}^N$ into two regions separated by a hyperplane $\{H^{\rm in}\in\mathbb{R}^N\colon W_{i:} H^{\rm in} = 0\}$ with normal $W_{i:}^\top\in\mathbb{R}^N$. 

\begin{figure}[t]
    \centering
\!\!\begin{tabular}{ccc}
GNN & SCNN & MPSN\\
\mbox{}
\!\!\! 
\includegraphics[clip=true,trim=5cm 8.2cm 4.5cm 7.5cm,width=.3\columnwidth]{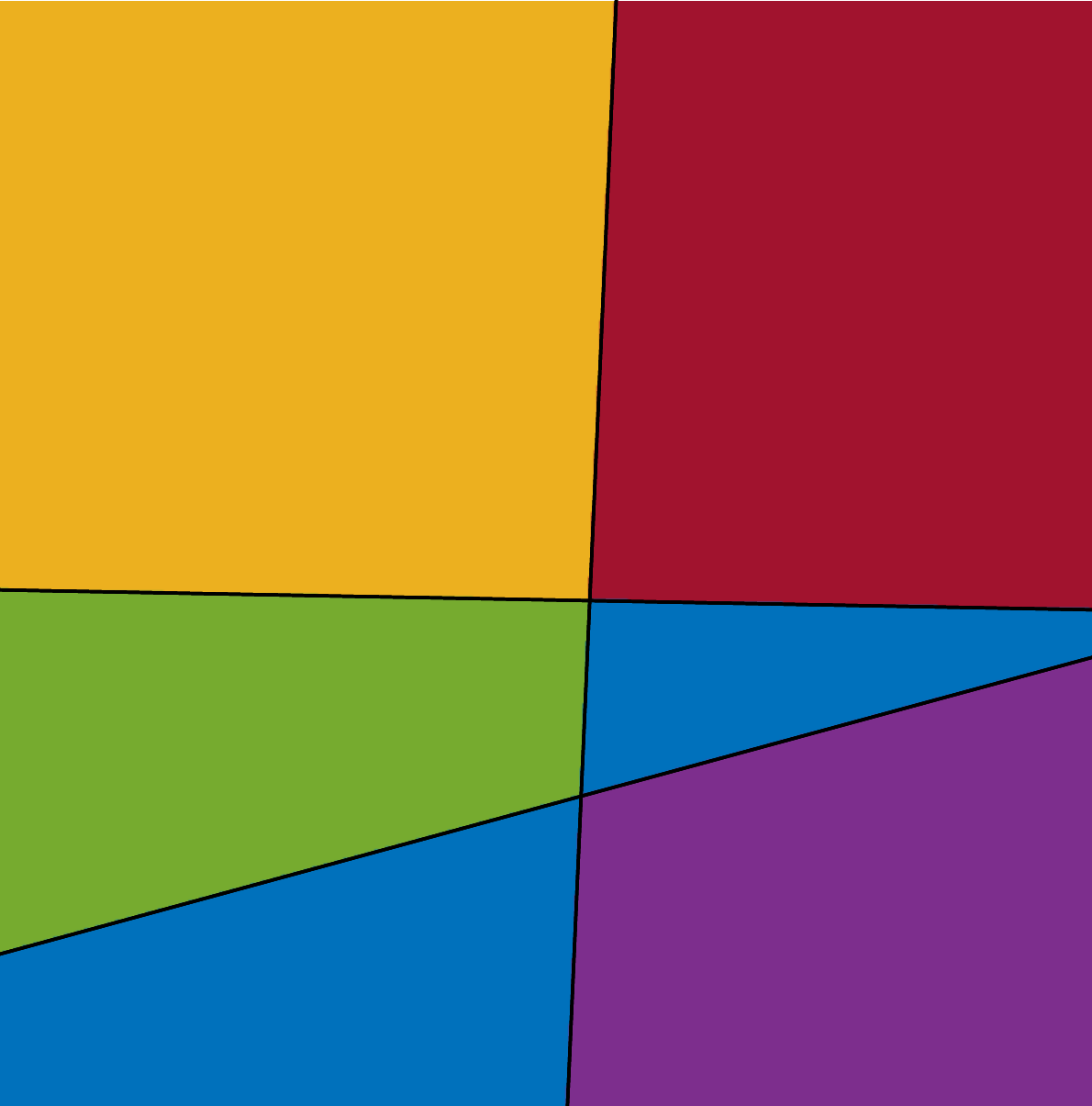}    
&\!\!\!
\includegraphics[clip=true,trim=5cm 8.2cm 4.5cm 7.5cm,width=.3\columnwidth]{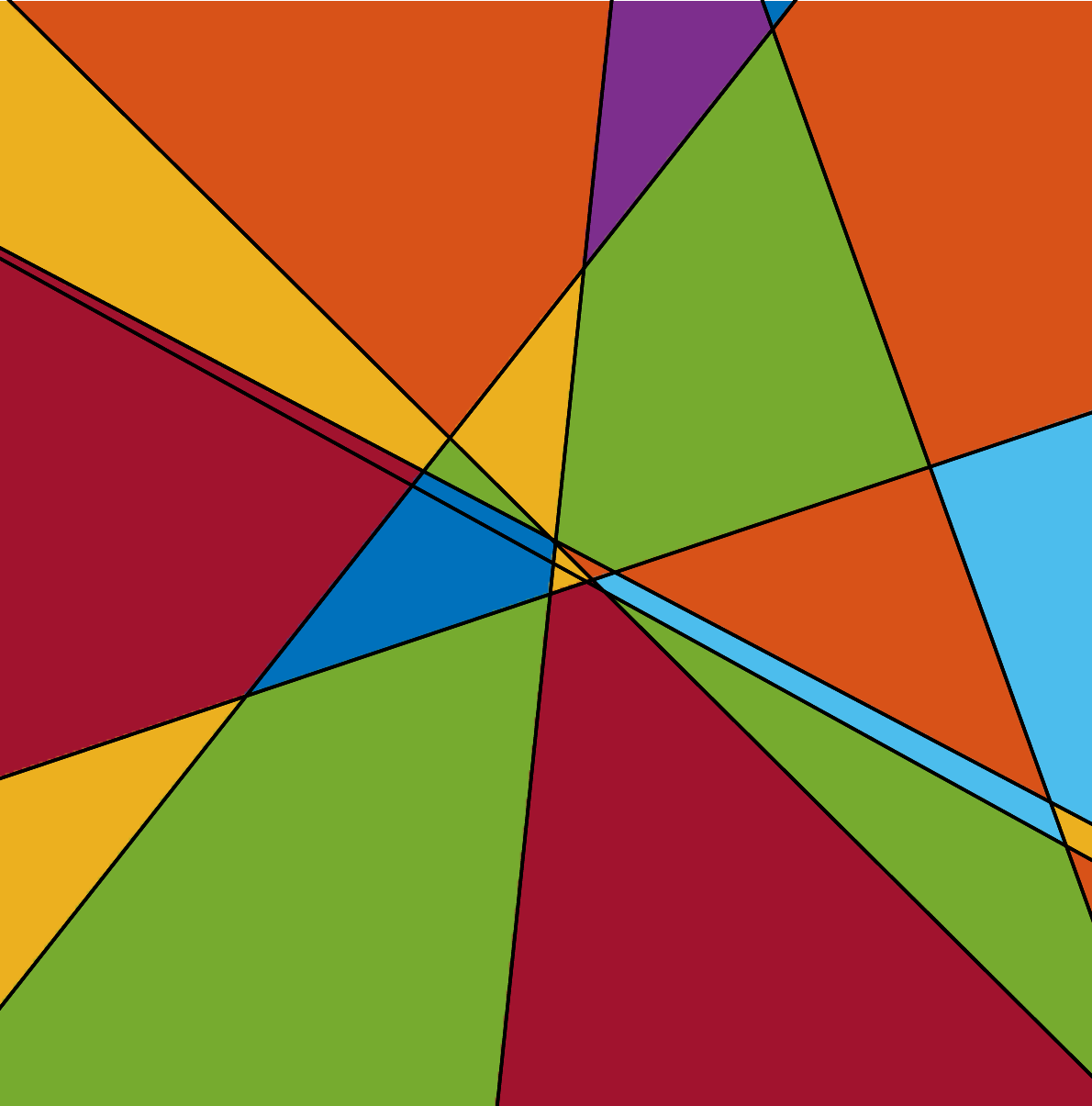}
&\!\!\! 
\includegraphics[clip=true,trim=5cm 8.2cm 4.5cm 7.5cm,width=.3\columnwidth]{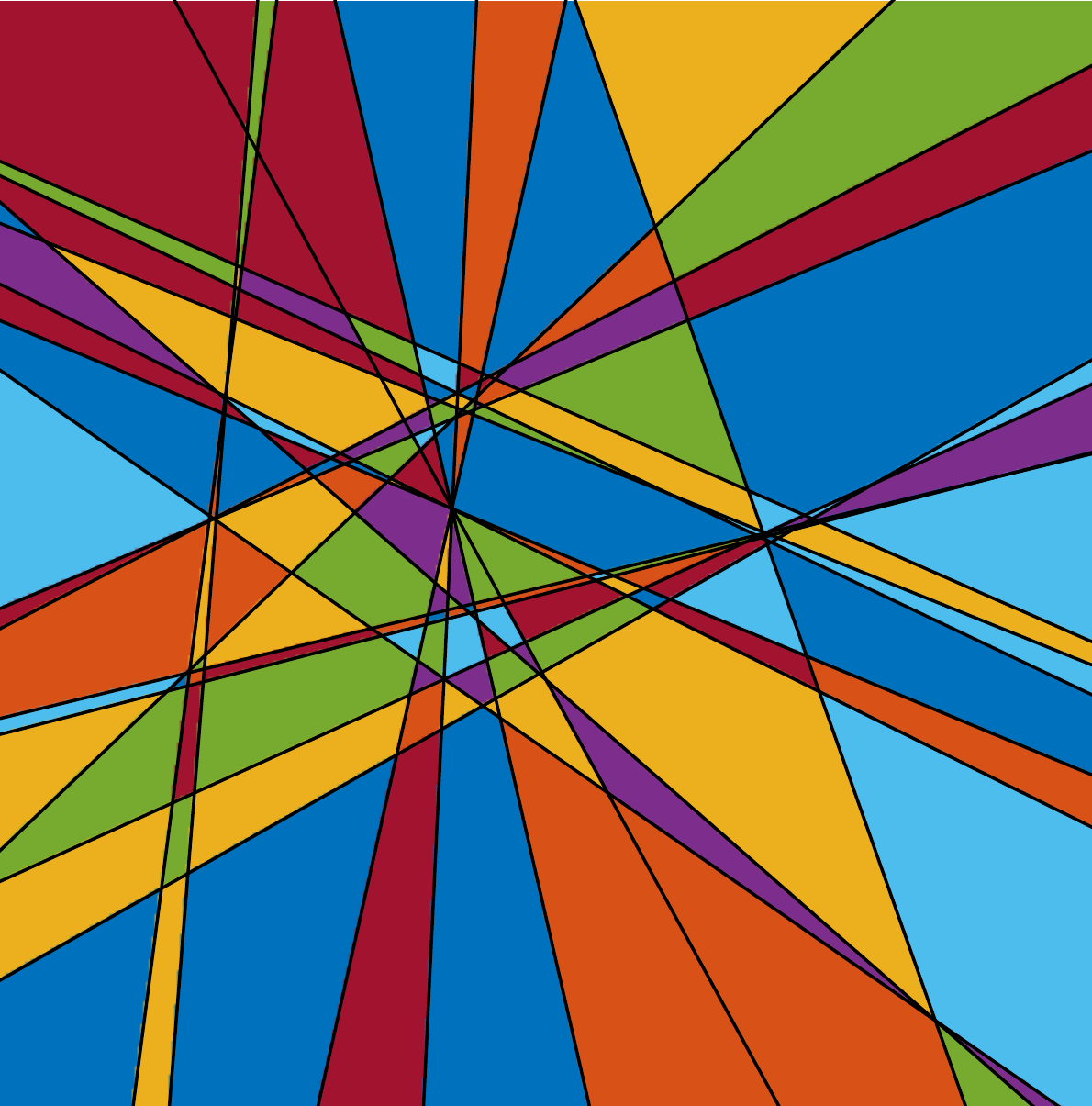} 
\end{tabular}
\vspace{-3pt}
\caption{A 2D slice of the input feature spaces of GNN, SCNN, MPSN layers with $S_0=S_1=3$, $S_2=1$ (the complex is a triangle), $d_0=d_1=d_2=1$, $m=3$, 
colored by linear regions of the represented functions, for a random choice of the weights. 
}\label{fig:regions}
\vspace{-1mm}
\end{figure}

In order to count the total number of regions, we will use results from the theory of hyperplane arrangements. 
\citet[Theorem~A]{zaslavsky1975facing} shows that the number of regions $r(\mathcal{A})$ defined by an arrangement $\mathcal{A}$ of hyperplanes in $\mathbb{R}^N$  
is 
$$
r(\mathcal{A}) = (-1)^{N} \chi_{\mathcal{A}}(-1),
$$
where $\chi_\mathcal{A}$ is the characteristic polynomial of the arrangement. 
%
By virtue of a theorem of Whitney (see \citealt[][Theorem~2.4]{Stanley04anintroduction} and \citealt[][Lemma 2.55]{orlik1992arrangements}),
it can be written as 
$
\chi_{\mathcal{A}}(t) = \sum
(-1)^{|\mathcal{B}|} t^{N - \operatorname{rank}(\mathcal{B})}, 
$ 
where the sum runs over subarrangements $\mathcal{B}\subseteq\mathcal{A}$ that are central (hyperplanes in $\mathcal{B}$ have a nonempty intersection), 
and $\operatorname{rank}(\mathcal{B})$ denotes the dimension spanned by the normals to the hyperplanes in $\mathcal{B}$. 
In our case, $\mathcal{A}$ is a central arrangement with normals given by the rows of the matrix $W$ in \eqref{eq:MPSNdimkron}. 
%
Hence: 
\begin{lemma}\label{lem:whitney}
The number of linear regions of the function \eqref{eq:MPSN full standard} with $W\in\mathbb{R}^{M\times N}$
and $\psi$ being ReLU is equal to 
\begin{equation*}
r(\mathcal{A}) = 
\sum_{B \subseteq \{1,\ldots, M\}} (-1)^{|B| - \operatorname{rank}(W_{B:})},
\end{equation*}
where $W_{B:}$ denotes the submatrix of rows $i\in B$. 
\end{lemma}
This formula counts the 
linear regions of any particular function represented by our layer. 
%
Some interesting cases can be computed explicitly. For instance: 
\begin{proposition}
\label{prop:simplecase}
Consider some $K\leq N$. If $\operatorname{rank}(W_{B:}) = \min\{|B|,K\}$ for any $B$, then $r(\mathcal{A}) =2\sum_{j=0}^{K-1}{M-1\choose j}$.  
\end{proposition}

We obtain the following bounds. 
\begin{theorem}[Number of linear regions of an MPSN layer]\label{thm:linear regions of MPSN}  With the above settings, the maximum number of linear regions of the functions represented by a ReLU MPSN layer \eqref{eq:MPSN full standard} is upper bounded by 
\begin{equation*}
 R_{\rm MPSN}  
 \leq \prod_{n=0}^p\left(2\sum_{i=0}^{d_{n-1}+d_n+d_{n+1}-1}{m - 1 \choose i}\right)^{S_n} ,  
\end{equation*}
where we set $d_{-1}=d_{p+1}=0$. 
We also note the `trivial' upper bound, with  $N:=\sum_{n=0}^pS_n d_n$ and $M:=\sum_{n=0}^p S_n m$,  
\begin{equation*}
R_{\rm MPSN}\leq 2\sum_{j=0}^{N-1}{M-1\choose j}.
\end{equation*}
Moreover, 
if $\rank((O_n)_{C:})\geq\rank((M_n)_{C:})$ for any selection $C$ of rows and $d_{n+1} \geq d_n$, for $n=0,\ldots, p-1$, 
then for networks with outputs $H_0^{\rm out},\ldots, H_{p-1}^{\rm out}$ we have 
\begin{equation}
R_{\rm MPSN} \geq R_{\rm SCNN}. 
 \label{eq:lower bound MPSN}
\end{equation}
\end{theorem}
By a more careful analysis of the rank conditions it is possible to obtain improvements of these bounds in specific cases, an endeavor that we leave for future work. 


We note that the MPSN lower bound \eqref{eq:lower bound MPSN} surpasses the SCNN upper bound \eqref{eq:upper bound SCNN}. 
The GNN bound \eqref{eq:upper bound GNN} is a special case of the SCNN bound \eqref{eq:upper bound SCNN} with $p=0$. 
The regions for the three network architectures are illustrated in Figure~\ref{fig:regions} for a complex with $S_0=S_1=3$ and $S_2=1$, each input dimension $1$ and output dimension $m=3$. It shows that from GNN, SCNN to MPSN the number of linear regions increases in turn, which is consistent with the theory.

\paragraph{MPSN with Populated Higher-Features}
We also consider a situation of interest where we are given a simplicial complex but only vertex features. To still exploit the structure of the simplicial complex, we can populate the higher features as linear functions of the vertex features. 
%
We show this strategy can increase the functional complexity, i.e.\ $R_{\rm MPSN} \geq R_{\rm SCNN}$. See Proposition~\ref{prop:ghost} 
in Appendix~\ref{app:linear regions}. 


\section{Experiments}
\label{sec:experiments}

\begin{figure}[t]
    \centering
    \vspace{-5pt}
    \includegraphics[width=0.9\columnwidth]{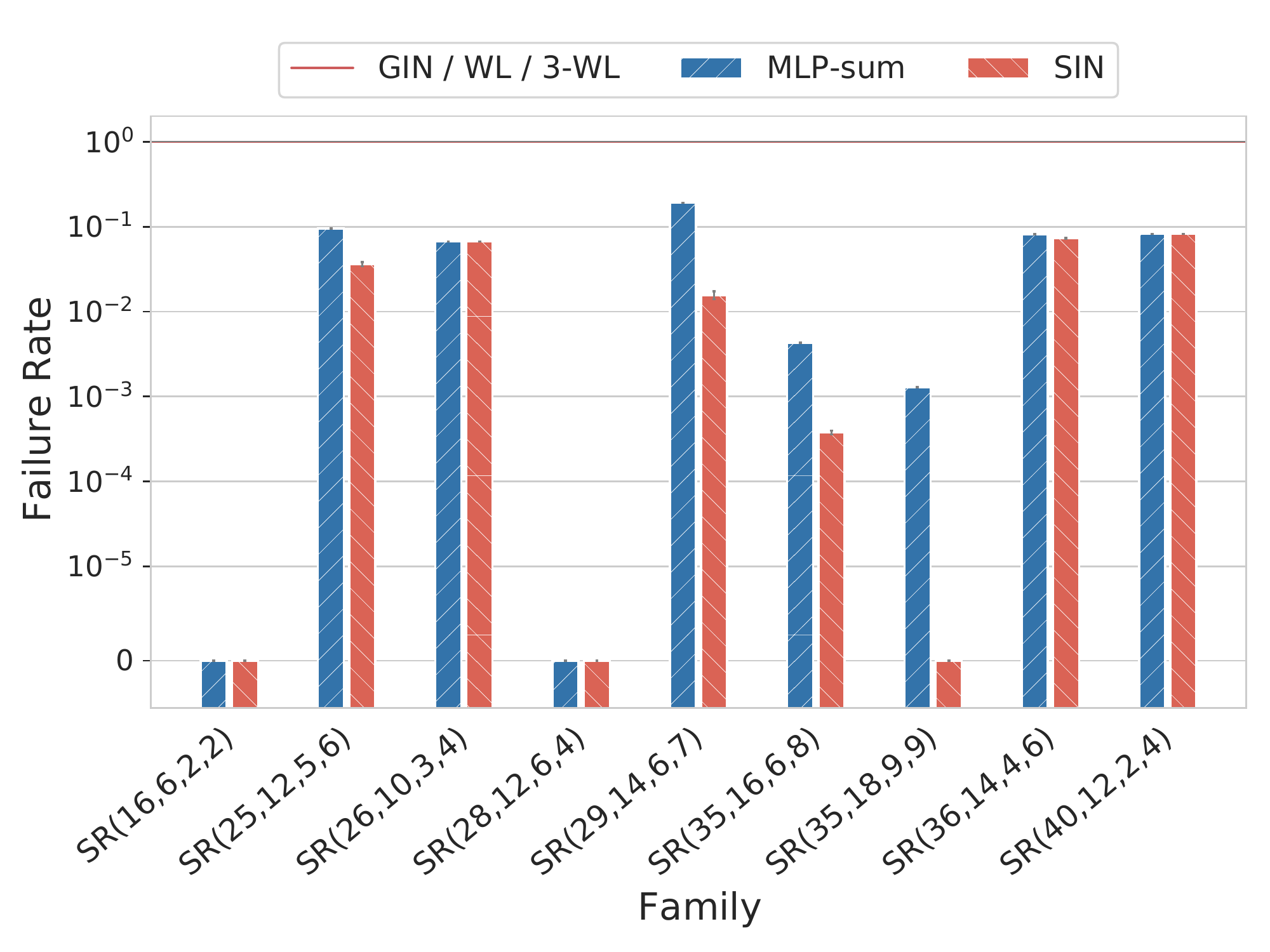}
    \vspace{-7pt}
    \caption{Failure rate on the task of distinguishing SR graphs; log-scale, \emph{the smaller the better}. GIN fails to distinguish all graph pairs in all families.}
    \label{fig:exp_sr}
    \vspace{-16pt}
\end{figure}

\paragraph{Strongly Regular Graphs}
We experimentally validate our theoretical result on the expressive power of our proposed architecture on the task of distinguishing hard pairs of non-isomorphic graphs. In particular, similarly to~\citet{bouritsas2020improving}, we benchmark it on $9$ synthetic datasets comprising families of Strongly Regular (SR) graphs. Strongly Regular graphs represent `hard' instances of graph isomorphism, as pairs thereof cannot provably be distinguished by the $3$-WL test (we refer readers to Section~\ref{app:higher_order_wl_and_srgs} for a formal proof).
In the experiments, we consider two graphs to be isomorphic if the Euclidean distance between their representations is below a fixed threshold $\varepsilon$. In particular, each graph is lifted to a $d$-dimensional simplicial complex, with $(d+1)$ the size of the largest clique in the family it belongs to. Lifted graphs are embedded by an untrained MPSN architecture parameterised similarly to GIN \citep{GIN}, which we refer to as Simplicial Isomorphism Network (SIN). Details are included in Appendix~\ref{app:exps}.

\begin{figure}[t]
    \centering
    \includegraphics[width=0.50\columnwidth]{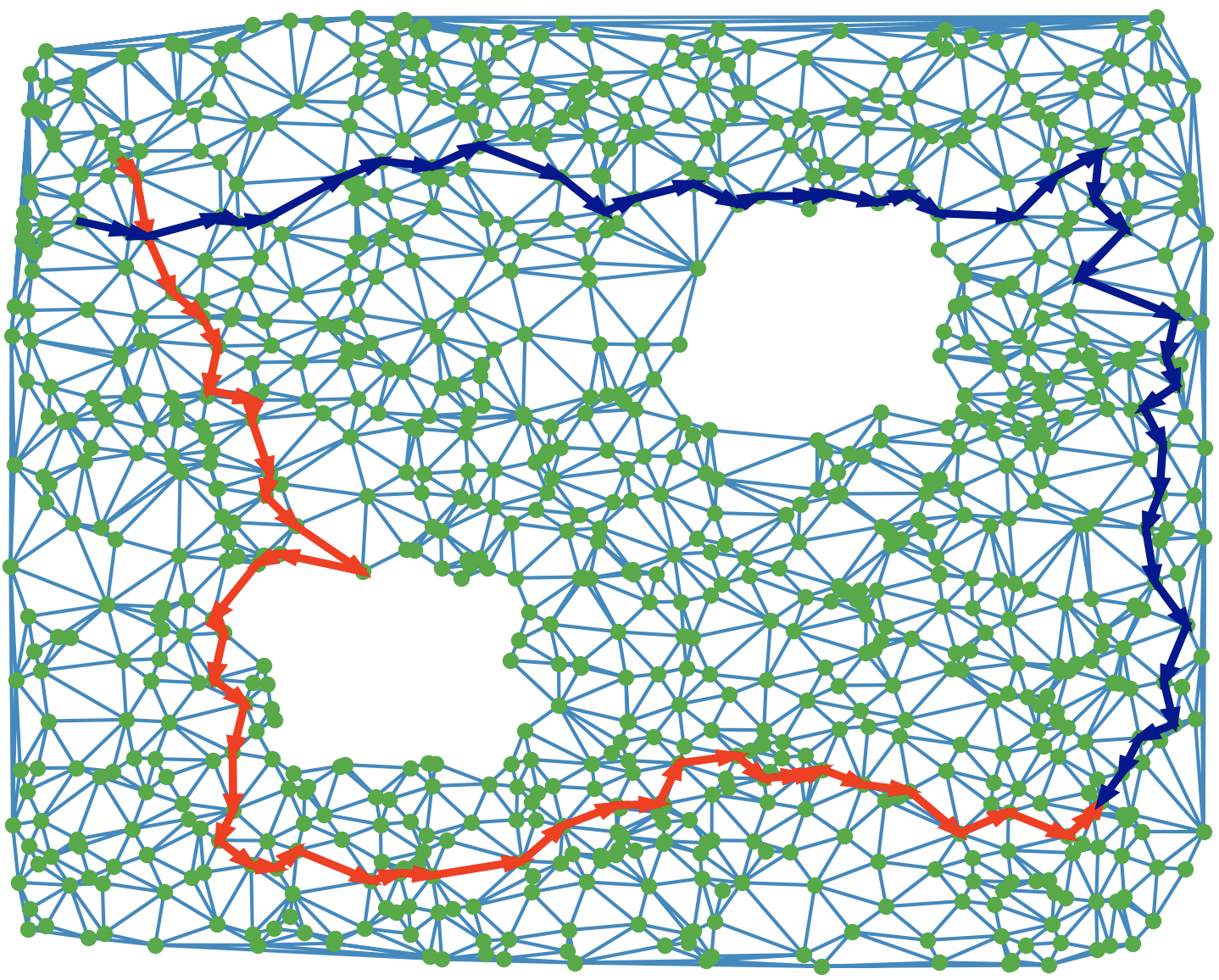}
    \caption{Two samples from the two different classes of trajectories. The two trajectories correspond to approximately orthogonal directions in the space of harmonic functions of $L_1$ associated with the two holes. }
    \label{fig:trajectory_complex}
    \vspace{-13pt}
\end{figure}
Results are illustrated in Figure~\ref{fig:exp_sr}, where we show performance on our isomorphism problem in terms of failure rate, that is the fraction of non-distinguished pairs. The experiment is performed for $10$ different random weight-initialisations and we report mean failure rate along with standard error (vertical error bars). We additionally report the performance of an MLP model with sum readout on the same inputs to assess the contribution of message passing to the disambiguation ability. As it can be observed, our (untrained) architecture is able to distinguish the majority of graph pairs in all families, since, in contrast to standard GNNs, it is able to access information related to the presence and number of cliques therein. Additionally, SIN outperforms the strong `MLP-sum' baseline on certain families, showing the favourable inductive bias intrinsic in simplicial message passing.

\paragraph{Edge-Flow Classification} We now turn our attention to two applications involving edge flows, which are represented as signals on oriented simplicial complexes. First, inspired by \citet{schaub2020random}, we generate a synthetic dataset of trajectories on the simplicial complex in Figure \ref{fig:trajectory_complex}, where we treat each triangle as a 2-simplex. All trajectories pass either through the bottom-left or the top-right corner, thus giving rise to two different classes that we aim to distinguish. Due to the two holes present in the complex, the trajectories of the two classes approximately correspond to orthogonal directions in the space of harmonic eigenfunctions of the $L_1$ Hodge-Laplacian \citep{schaub2020random}. Therefore, we hypothesise, that an orientation invariant MPSN network with orientation equivariant layers should be able to distinguish the two classes. The dataset contains 1000 train trajectories and 200 test trajectories. To make the task more challenging for non orientation invariant models, all the training complexes use the same orientation for the edges, while the test trajectories use random orientations. More details are in Appendix \ref{app:exps}.

Additionally, we consider a real-world equivalent of the synthetic benchmark above. Again, we adapt a benchmark from \citet{schaub2020random} containing ocean drifter trajectories around the island of Madagascar between years 2011-2018. To obtain a simplicial complex, we discretise the surface of the map into a simplicial complex containing a hole in the center, representing the island. The task is to distinguish between the clockwise and counter-clockwise flows around the island. As in the previous benchmark, the presence of the hole makes the harmonic signal associated with it extremely important for solving the task. The dataset has 160 train trajectories and 40 test trajectories. As before, the test flows use random orientations for each trajectory to make the task more difficult for non-invariant models. 

\begin{table}[t]
    \centering
    \begin{minipage}[t]{1.0\columnwidth}
        \centering
         \caption{Trajectory classification accuracy. Models with triangle awareness and orientation equivariance generalise better.}
        \label{tab:traj_classification}
           \resizebox{\columnwidth}{!}{
          \begin{tabular}{l cccc}
            \toprule
            \multirow{2}{*}{Method} & 
            
            \multicolumn{2}{c}{Synthetic Flow} &
            \multicolumn{2}{c}{Ocean Drifters} \\
            
            &
            Train &
            Test &
            Train &
            Test \\
            \midrule
            
            GNN $L_0$-inv & 
            63.9$\pm$2.4 &
            61.0$\pm$4.2 &
            70.1$\pm$2.3 &
            63.5$\pm$6.0 \\
            
            MPSN $L_0$-inv  &
            88.2$\pm$5.1 &
            85.3$\pm$5.8 &
            84.6$\pm$4.0 &
            \textbf{71.5}$\pm$\textbf{4.1} \\
            
            MPSN - ReLU &
            \hspace{-5pt}\textbf{100.0}$\pm$\textbf{0.0} & 
            50.0$\pm$0.0 &
            \hspace{-5pt}\textbf{100.0}$\pm$\textbf{0.0}  &
            46.5$\pm$5.7 \\
            
            MPSN - Id & 
            88.0$\pm$3.1 &
            82.6$\pm$3.0 &
            94.6$\pm$0.9 &
            \textbf{73.0}$\pm$\textbf{2.7} \\

            MPSN - Tanh  & 
            97.9$\pm$0.7 &
            \textbf{95.2}$\pm$\textbf{1.8} &
            \textbf{99.7}$\pm$\textbf{0.5} &
            \textbf{72.5}$\pm$\textbf{0.0} \\
        
            \bottomrule
          \end{tabular}%
          }
    \end{minipage}
    \vspace{-10pt}
\end{table}

We evaluate multiple MPSN models with lower and upper adjacencies, each being orientation invariant. The first model, MPSN $L_0$-inv, is made invariant from the first layer by simply using the absolute value of the features and ignoring the relative orientations between simplices. Two other MPSN models have equivariant layers like in Theorem \ref{theo:mpsn_orient_equiv} and use odd activation functions (Id and Tanh). They end with a final permutation invariant readout layer. The last MPSN model is similar to the previous two, but uses a ReLU activation function, which breaks its invariance. We also consider a GNN baseline operating in the line graph, unaware of the triangles. Like the first MPSN baseline above, it is made invariant by using absolute values. 

Results for both benchmarks over $5$ seeds are shown in Table~\ref{tab:traj_classification}. We notice that the GNN's unawareness of the triangles makes it perform worse, since it cannot extract the harmonic part of the signal. The $L_0$-inv model performs well on both benchmarks, but it generally lags behind the models that have  equivariant layers (i.e.\ Id and Tanh) and use a final invariant layer. Among these last two, wee see that the non-linear model generally performs better. Finally, we remark the ReLU model perfectly fits the training dataset of both benchmarks, which use a fixed orientation, but does not generalise to the random orientations of the test set.    


\paragraph{Real-World Graph Classification}

\begin{table}[t!]
    \centering
    \caption{Graph classification results on the TUDatasets benchmark. The table contains: dataset details \emph{(top)}, graph kernel methods \emph{(middle)}, and graph neural networks \emph{(bottom)}.}
    \label{tab:tud_datasets}
    \resizebox{1.0\linewidth}{!}{%
    \begin{tabular}{l|cccccc}
        \toprule
        Dataset & Proteins & NCI1 & IMDB-B & IMDB-M & RDT-B & RDT-M5K\\
        \midrule
        Avg $\bigtriangleup$ &
            27.4 &
            0.05 &
            392.0 &
            305.9 &
            24.8 &
            21.8\\
        Med $\bigtriangleup$ &
            21.0 &
            0.0 &
            119.5 &
            56.0 &
            11.0 &
            11.0\\
        \midrule       
        RWK 
        & 
         59.6$\pm$0.1 & 
         $>$3 days & 
         N/A &
         N/A &
         N/A &
         N/A\\
        
        GK (k=3) 
        &
        71.4$\pm$0.31 & 
        62.5$\pm$0.3 & 
        N/A &
        N/A &
        N/A &
        N/A\\

        PK 
        & 
         73.7$\pm$0.7 & 
         82.5$\pm$0.5 & 
         N/A & 
         N/A &
         N/A &
         N/A\\

          WL kernel 
          &
          75.0$\pm$3.1 & 
          86.0$\pm$1.8 & 
          73.8$\pm$3.9 &
          50.9$\pm$3.8 &
          81.0$\pm$3.1 &
          52.5$\pm$2.1\\
        
        \midrule
         
        DCNN 
        & 
          61.3$\pm$1.6
          & 56.6$\pm$1.0 &
          49.1$\pm$1.4 &
          33.5$\pm$1.4 &
          N/A &
          N/A\\

         DGCNN 
         & 
        75.5$\pm$0.9 & 
        74.4$\pm$0.5 & 
        70.0$\pm$0.9 & 
        47.8$\pm$0.9 &
        N/A &
        N/A\\

        IGN 
        & 
        76.6$\pm$5.5 &
        74.3$\pm$2.7 & 
        72.0$\pm$5.5 & 
        48.7$\pm$3.4 &
        N/A &
        N/A\\
        
        GIN 
        & 
        76.2$\pm$2.8 &
        82.7$\pm$1.7 &
        75.1$\pm$5.1 &
        52.3$\pm$2.8 &
        92.4$\pm$2.5 &
        57.5$\pm$1.5\\
        
        PPGNs 
        &
        77.2$\pm$4.7 & 
        83.2$\pm$1.1 & 
        73.0$\pm$5.8 & 
        50.5$\pm$3.6 &
        N/A &
        N/A\\
        Natural GN 
        &
        71.7$\pm$1.0 &
        82.4$\pm$1.3 &
        73.5$\pm$2.0 &
        51.3$\pm$1.5 &
        N/A&
        N/A\\
        
        GSN 
        &
        76.6 $\pm$ 5.0 & 
        83.5 $\pm$ 2.0 &
        77.8 $\pm$ 3.3 & 
        54.3 $\pm$ 3.3 &
        N/A &
        N/A \\
        
        \midrule
        
        {\bf SIN (Ours)} & 
        76.5 $\pm$ 3.4 & 
        82.8 $\pm$ 2.2 &
        75.6 $\pm$ 3.2 & 
        52.5 $\pm$ 3.0 &
        92.2 $\pm$ 1.0 &
        57.3 $\pm$ 1.6\\
        
        \bottomrule
        
    \end{tabular}
    }
    \vspace{-15pt}
\end{table}

Finally, we study the practical impact of considering higher-order interactions via (non-oriented) clique-complexes and report results for a few popular graph classification tasks commonly used for benchmarking GNNs~\citep{morris2020tu}. 
We follow the same experimental setting and evaluation procedure described in \citet{GIN}. Accordingly, we report the best mean test accuracy computed in a $10$-fold cross-validation fashion. We employ a SIN model similar to that employed in the SR graph experiments. We lift the original graphs to $2$-complexes by considering $3$-cliques (triangles) as $2$-simplices (see Appendix~\ref{app:exps} for more details).
The performance of SIN are reported in Table~\ref{tab:tud_datasets}, 
along with those of graph kernel methods (RWK \citep{gartner2003graph},
GK 
\citep{shervashidze2009efficient},
PK \citep{neumann2016propagation},
WL kernel \citep{shervashidze2011weisfeiler}) and other GNNs (DCNN \citep{DCNN_2016}, DGCNN \citep{zhang2018end}, IGN \citep{maron2018invariant}, GIN \citep{GIN}, PPGNs \citep{maron2019provably}, Natural GN \citep{de2020natural}, GSN \citep{bouritsas2020improving}). We observe that our model achieves its best results on the IMDB datasets, which have the largest mean and median number of triangles. In contrast, on datasets like NCI1, where the number of higher-order structures is close to zero, the model shows the same mean accuracy as GIN. Overall, we observe SIN to perform on-par with other GNN approaches. 

\section{Discussion and Conclusion}

\paragraph{Provably Powerful GNNs} In order to overcome the limited expressive power of standard GNN architectures, several works have proposed variants inspired by the higher-order $k$-WL procedures (see Appendix). \citet{maron2019provably} introduced a model equivalent in power to $3$-WL, \citet{morris2019weisfeiler} proposed $k$-GNNs, graph neural networks equivalents of set-based $k$-WL tests. By performing message passing on all possible $k$-tuples of nodes and across non-local neighborhoods, these models trade locality of computation for expressive power, and thus suffer from high spatial and temporal complexities. Local $k$-WL variants were introduced by~\citet{morris2020weisfeiler}, and distinguish local and global neighbors. The authors also propose provably powerful neural counterparts. Although more efficient, in contrast to our method, this approach still accounts for all possible node $k$-tuples in a graph. An alternative approach to improving GNN expressivity has been adopted in~\citep{bouritsas2020improving}, where isomorphism counting of graph substructures is employed as a symmetry breaking mechanism to disambiguate neighbours. Similarly to ours, this approach retains locality of operations; however, message passing is only performed at the node level.

\paragraph{Beyond Pairwise Interactions} We note that other graph lifting transformations could also be used for applying SWL and MPSNs to graph domains. While clique complexes are the commonest such transformation, many others exist \citep{ferrisimplicial} and they could be used to emphasise 
motifs that are relevant for the task \citep{Milo824}. More broadly, the transformation could target a much wider output space such as cubical complexes \citep{kac04} (see Appendix \ref{app:cubical}), or cell complexes \citep{hatcher_book}, for which a message passing procedure has already been proposed \citep{hajij2020cell}. One can study even more flexible structures described by incidence tensors \citep{albooyeh2020incidence}, which can also encode simplicial complexes. Alternatively, if the goal is to learn higher-order representations, one could also do so in an unsupervised manner directly from pairwise interactions~\citep{cotta2020unsupervised} or by explicitly modelling subgraphs~\citep{Alsentzer_subgraph}. We conclude this part by mentioning a large body of work on neural networks for hypergraphs, which subsume simplicial complexes \citep{hypergraph_gcn, feng2019hypergraph, Zhang2020Hyper-SAGNN:, Yadati_recursive_hypergraphs}. However, none of these study their expressive power in relationship to the WL hierarchy.  


\paragraph{Conclusion}
We introduce a provably powerful message passing procedure for simplicial complexes relying on local higher-order interactions. We motivate our message passing framework by the introduction of SWL, a colouring algorithm for simplicial complex isomorphism testing, generalising the WL test for graphs. We prove that when graphs are lifted in the simplicial complex space via their clique complex, SWL and MPSNs are more expressive than the WL test. We also analyse MPSNs through the lens of symmetry by showing they are simplex permutation equivariant and can be made orientation equivariant. Furthermore, we produce an estimate for the number of linear regions of GNNs, SCNNs and MPSNs, which also reveals the superior expressive power of our model. 
We empirically confirm these results by distinguishing between challenging non-isomorphic SR graphs, on real-world graph classification benchmarks, on edge flow classification tasks, and by computing 2D slices through the linear regions of the models. 

\section*{Acknowledgements}
YW and GM acknowledge support from the ERC under the EU's Horizon 2020 
programme (grant agreement n\textsuperscript{o} 757983). MB is supported in part by ERC Consolidator grant n\textsuperscript{o} 724228 (LEMAN). We would also like to thank Teodora Stoleru for creating Figure \ref{fig:sin} and Ben Day for providing valuable comments on an early draft. We are also grateful to the anonymous reviewers for their  feedback. 

\FloatBarrier
\bibliographystyle{icml2021}
\bibliography{gnn}

\cleardoublepage
\appendix

\allowdisplaybreaks

\twocolumn[
\begin{center}
    \Large{\textbf{Weisfeiler and Lehman Go Topological: Message Passing Simplicial Networks}}
\end{center}

\begin{center}
    \Large{\textbf{Appendix}}
\end{center}
\vspace{20pt}
]

\section{Proofs of SWL Theory Results}
\label{sec:proofs}

We first introduce the required notions and notation. We note that even though these results mainly refer to simplicial complexes, they also apply to graphs because any graph is also a simplicial complex.  

\begin{definition}[Simplicial Colouring]
A simplicial colouring is a map $c$ that maps a simplicial complex $\gK$  and one of its simplices $\sigma$ to a colour from a fixed colour palette. We denote this colour by $c^\gK_\sigma$. 
\end{definition}

To unload the notation, we will often drop $\gK$ from the superscript when the underlying $\gK$ is arbitrary. 

\begin{definition}
Let $\gK_1, \gK_2$ be two simplicial complexes and $c$ a simplicial colouring. We say that $\gK_1, \gK_2$ are $c$-similar, denoted by $c^{\gK_1} = c^{\gK_2}$, if the number of simplices of dimension $n$ in $\gK_1$ coloured with a given colour equals the number of simplices of dimension $n$ in $\gK_2$ with the same colour, for all $n$. Otherwise, we have $c^{\gK_1} \neq c^{\gK_2}$.  
\end{definition}

The reason we treat simplices of different dimensions separately is only to simplify the proofs. This is not strictly required because it can be shown SWL can identify the dimension of each simplex after a sufficient number of steps. Therefore, with some careful modifications to the proofs, the definition can be relaxed. 

Additionally, although not explicitly stated in the definition of a simplicial colouring, we are only interested in colourings $c$ for which all isomorphic simplicial complex pairs are $c$-similar or, more formally, if $\gK_1$ is isomorphic to  $\gK_2$, then $c^{\gK_1} = c^{\gK_2}$. 

\begin{definition}
A simplicial colouring $c$ refines a simplicial colouring $d$, denoted by $c \sqsubseteq d$, if for all simplicial complexes $\gK_1$ and $\gK_2$
and all $\sigma \in \gK_1$ and $\tau \in \gK_2$ with $\mathrm{dim}(\sigma) = \mathrm{dim}(\tau)$, $c^{\gK_1}_\sigma = c^{\gK_2}_\tau$ implies $d^{\gK_1}_\sigma = d^{\gK_2}_\tau$. Additionally, if $d \sqsubseteq c$, we say the two colourings are equivalent and we represent it by $c \equiv d$. 
\end{definition}

We now prove a lemma that will be used repeatedly in our proofs in this section. 

\begin{lemma}
\label{lemma:refine_multiset}
Let $\gK_1, \gK_2$ be any simplicial complexes with $A \subseteq \gK_1$ and $B \subseteq \gK_2$, two subsets containing simplices of the same dimension. Consider two simplicial colourings $c, d$ such that $c \sqsubseteq d$. If $\ldblbrace d_\sigma^{\gK_1} \mid \sigma \in A \rdblbrace \neq \ldblbrace d_\tau^{\gK_2} \mid \tau \in B \rdblbrace$, then $\ldblbrace c_\sigma^{\gK_1} \mid \sigma \in A \rdblbrace \neq \ldblbrace c_\tau^{\gK_2} \mid \tau \in B \rdblbrace$.
\end{lemma}

\begin{proof}
If $\ldblbrace d_\sigma^{\gK_1} \mid \sigma \in A \rdblbrace \neq \ldblbrace d_\tau^{\gK_2} \mid \tau \in B \rdblbrace$, there exists a colour $\sC$ that shows up (without loss of generality) more times in the first multi-set than in the second multi-set. Define the sets 
$$A^* = \{ \sigma \in A \mid d_\sigma^{\gK_1} = \sC \}, \quad B^* = \{ \tau \in B \mid d_\tau^{\gK_2} = \sC \},$$
of simplices from $A$ and $B$, respectively, that have been assigned colour $\sC$ by $d$. Note that because $\sC$ shows up more times in the first multi-set, $\vert A^* \vert > \vert B^* \vert$. 

Since $c \sqsubseteq d$, if $d_\sigma \neq d_\tau$, then $c_\sigma \neq c_\tau$, for any $\sigma$ and $\tau$. Therefore, the $c$-colouring of all the simplices in $A^*$ and $B^*$ is different from the colouring of the simplices outside $A^*$ and $B^*$. More formally, 
\begin{align}\label{eq:empty_intersect}
    \ldblbrace &c_\sigma \mid \sigma \in A^* \cup B^* \rdblbrace \nonumber \\
    \qquad &\cap \ldblbrace c_\tau \mid \tau \in (A \cup B) \setminus (A^* \cup B^*) \rdblbrace = \emptyset.
\end{align}
Suppose for the sake of contradiction that $$\ldblbrace c_\sigma^{\gK_1} \mid \sigma \in A \rdblbrace = \ldblbrace c_\tau^{\gK_2} \mid \tau \in B \rdblbrace.$$ 
Together with Equation \ref{eq:empty_intersect}, this implies
$$\ldblbrace c_\sigma^{\gK_1} \mid \sigma \in A^* \rdblbrace = \ldblbrace c_\tau^{\gK_2} \mid \tau \in B^* \rdblbrace.$$
Then, $\vert A^* \vert = \vert B^* \vert$. However, $\vert A^* \vert > \vert B^* \vert$.
\end{proof}

This result leads to an important corollary. 

\begin{corollary}
Consider two simplicial colourings $c, d$ such that $c \sqsubseteq d$. If $d^{\gK_1} \neq d^{\gK_2}$, then $c^{\gK_1} \neq c^{\gK_2}$
\end{corollary}

\begin{proof}
This follows immediately by substituting the subsets $A, B$ from the proof above with all the simplices of a given dimension in $\gK_1$ and $\gK_2$, respectively. 
\end{proof}

An equivalent way to think about this corollary is that if $c$ refines $d$, then it is able to distinguish all the non-isomorphic simplicial complex pairs that $d$ can distinguish (and potentially others). In that sense, we say $c$ is \emph{at least as powerful as} $d$. This will be our main vehicle to prove the results. 

Equipped with this notation and preliminary results, we proceed to prove the results from the main text. 

\begin{lemma}
\label{lemma:drop_cofaces}
    SWL with $\mathrm{HASH}\bigl(c_\sigma^t, c_{\gB}^t(\sigma), c_{\da}^t(\sigma), c_{\ua}^t(\sigma)\bigr)$ is as powerful as SWL with the generalised update rule $\mathrm{HASH}\bigl(c_\sigma^t, c_{\gB}^t(\sigma), c_{\gC}^t(\sigma), c_{\da}^t(\sigma), c_{\ua}^t(\sigma)\bigr)$.
\end{lemma}

\begin{proof}
Let $a^t$ denote the colouring of the general update rule at iteration $t$ and $b^t$ the colouring of the restricted update rule at the same iteration. Then, $a^t \sqsubseteq b^t$ because it considers the additional colours of the co-boundaries $c_{\gC}^t(\sigma)$ in the colour updating rule. We will now prove by induction $b^t \sqsubseteq a^t$, which implies $a^t \equiv b^t$.

The base case trivially holds since all simplices have the same colour at initialisation. Let $\sigma \in \gK_1$ and $\tau \in \gK_2$ be two simplices of the same dimension from two arbitrary complexes. Suppose $b_\sigma^{t+1} = b_\tau^{t+1}$. Then, we have that the arguments of the hash function are equal. Thus, $b_\sigma^t = b_\tau^t, b_\da^t(\sigma) = b_\da^t(\tau), b_\ua^t(\sigma) = b_\ua^t(\tau)$, and $b_\gB^t(\sigma) = b_\gB^t(\tau)$. The aim is to show that these also imply that $b_\gC^t(\sigma) = b_\gC^t(\tau)$. 

Because $b_\ua^t(\sigma) = b_\ua^t(\tau)$, by substituting their definition we have the following equality of multi-sets. 
$$\ldblbrace b_{\delta_\sigma}^{t} \mid (\cdot, b_{\delta_\sigma}^{t}) \in b_\ua^{t}(\sigma) \rdblbrace
= \ldblbrace b_{\delta_\tau}^{t} \mid (\cdot, b_{\delta_\tau}^{t}) \in b_\ua^{t}(\tau) \rdblbrace.$$ 
Because $\sigma$ and $\tau$ have the same dimension $n$, the colour of each $\delta_\sigma \in \gC(\sigma)$ and each $\delta_\tau \in \gC(\tau)$ shows up in exactly $n + 1$ tuples in $b_\ua^{t}(\sigma)$ and $b_\ua^{t}(\tau)$, respectively. By removing the duplicate colours for each such $\delta_\sigma$ and $\delta_\tau$ we obtain the desired equality: 
$$\ldblbrace b_{\delta_\sigma}^{t} \mid \delta_\sigma \in \gC(\sigma) \rdblbrace
= \ldblbrace b_{\delta_\tau}^{t} \mid \delta_\tau \in \gC(\tau) \rdblbrace.$$ 
By the induction hypothesis, we also have $a_\sigma^t = a_\tau^t, a_\da^t(\sigma) = a_\da^t(\tau), a_\ua^t(\sigma) = a_\ua^t(\tau)$ , $a_\gB^t(\sigma) = a_\gB^t(\tau)$, and $a_\gC^t(\sigma) = a_\gC^t(\tau)$. Thus, $a_v^{t+1} = a_w^{t+1}$. 
\end{proof}

\begin{proof}[Proof of Theorem~\ref{thm:sparse swl}]
Let $b^t$ denote the colouring of CWL using $\mathrm{HASH}\bigl(b_\sigma^t, b_{\gB}^t(\sigma), b_{\ua}^t(\sigma)\bigr)$ and $a^t$ the colouring of CWL using the rule $\mathrm{HASH}\bigl(a_\sigma^t, a_{\gB}^t(\sigma), a_{\da}^t(\sigma), a_{\ua}^t(\sigma)\bigr)$ from  Lemma~\ref{lemma:drop_cofaces}. As before, it is trivial to show $a^t \sqsubseteq b^t$ because of the additional argument $a_{\da}^t(\sigma)$ used in the update rule. We now prove that $b^{2t} \sqsubseteq a^t$ by induction. The reason we consider $2t$ is because the information from the lower adjacencies propagates two times slower through the boundary adjacencies. 

As before, the base case holds since all the colours are equal at initialisation. Again, consider $\sigma \in \gK_1$ and $\tau \in \gK_2$ , two simplices of the same dimension from two arbitrary complexes. Suppose $b_\sigma^{2t+2} = b_\tau^{2t+2}$. By unwrapping the hash function two steps back in time, we obtain $b_\sigma^{2t} = b_\tau^{2t}, b_{\gB}^{2t}(\sigma) = b_{\gB}^{2t}(\tau), b_{\ua}^{2t}(\sigma) = b_{\ua}^{2t}(\tau)$. The goal is to show that $b_\da^{2t}(\sigma) = b_\da^{2t}(\tau)$ also holds. 

Suppose for the sake of contradiction that $b_\da^{2t}(\sigma) \neq b_\da^{2t}(\tau)$. This means that there exists a pair of colours $(\sC_0, \sC_1)$ that shows up (without loss of generality) more times in $b_\da^{2t}(\sigma)$ than in $b_\da^{2t}(\tau)$. For simplicity, we assume that $b_\sigma^{2t} \neq \sC_0 \neq b_\tau^{2t}$, since this edge case can be trivially treated separately. 

First, we split the apparitions of $(\sC_0, \sC_1)$ by the boundary simplices of $\sigma$ and $\tau$ where they appear. Consider the collection of multi-sets $A$ indexed by simplices $\delta$ of a fixed dimension: 
$$A(\delta) = \ldblbrace (b_\psi^{2t} = \sC_0, b_\delta^{2t} = \sC_1) \mid \psi \in \gC(\delta) \rdblbrace.$$
We are interested in counting the size of these multi-sets. For this purpose, for each simplex $\gamma$, we define a multi-set $C_\gamma$:
$$C_\gamma = \ldblbrace \vert A(\delta) \vert \mid \delta \in \gB(\gamma) \rdblbrace.$$
Clearly, $C_\sigma \neq C_\tau$ because the sum of the elements in $C_\sigma$ (the number of tuples $(\sC_0, \sC_1)$ in $b_\da^{2t}(\sigma)$) is greater than the sum of the elements of $C_\tau$ (the number of tuples $(\sC_0, \sC_1)$ in $b_\da^{2t}(\tau)$). The next proposition, shows this leads to a contradiction with our original assumption that $b_\sigma^{2t+2} = b_\tau^{2t+2}$

\begin{proposition}
If $C_\sigma \neq C_\tau$, then $b_\sigma^{2t+2} \neq b_\tau^{2t+2}$.
\end{proposition}

\begin{proof}
Consider the simplicial colouring $c(\delta) = \vert A(\delta) \vert$. We will show that $b^{2t+1} \sqsubseteq c$. Let $\delta_1, \delta_2$ be two simplices of equal dimension with $c(\delta_1) \neq c(\delta_2)$. We assume without loss of generality $\vert A(\delta_1) \vert > \vert A(\delta_2) \vert$. Then $\sC_0$ shows up more times in $b_\ua^{2t}(\delta_1)$ than in $b_\ua^{2t}(\delta_2)$, which implies $b_\ua^{2t}(\delta_1) \neq b_\ua^{2t}(\delta_2)$. Therefore, $b^{2t+1}_{\delta_1} \neq b^{2t+1}_{\delta_2}$, which proves $b^{2t+1} \sqsubseteq c$.

Applying Lemma \ref{lemma:refine_multiset} for the multi-sets $C_\sigma$ and $C_\tau$, we obtain two non-equal multi-sets:
$$\ldblbrace b_{\delta_1}^{2t+1} \mid \delta_1 \in \gB(\sigma) \rdblbrace \neq \ldblbrace b_{\delta_2}^{2t+1} \mid \delta_2 \in \gB(\tau) \rdblbrace$$
Which are exactly the multi-sets of colours corresponding to the boundary simplices of $\sigma$ and $\tau$. So the relation above is equivalent to $b_\gB^{2t+1}(\sigma) \neq b_\gB^{2t+1}(\tau)$. Finally, this implies that $b_\sigma^{2t+2} \neq b_\tau^{2t+2}$. 
\end{proof}

This contradiction proves $b_\da^{2t}(\sigma) = b_\da^{2t}(\tau)$. Finally, applying the induction hypothesis, we have that $a_v^{t} = a_w^{t}, a_{\gB}^{t}(v) = a_{\gB}^{t}(w), a_{\ua}^{t}(v) = a_{\ua}^{t}(w)$ and  $a_\da^t(v) = a_\da^t(w)$. Then, $b^{2t} \sqsubseteq a^t$. 
\end{proof}

Next, we show a slightly weaker version of Theorem~\ref{theo:swl_more_powerful_than_wl}. 

\begin{lemma}\label{lemma:wl_is_at_most_swl}
SWL is at least as powerful as WL in distinguishing non-isomorphic simplicial complexes. 
\end{lemma} 
\begin{proof}
Let $\gK$ be a simplicial complex. Let $a^t$ be the colouring of the vertices of $\gK$ at iteration $t$ of WL and $b^t$ the colouring of the same vertices in $\gK$ at iteration $t$ of SWL. To prove the lemma, we will show by induction that $b^t \sqsubseteq a^t$. 

For the base case, the implication holds at initialisation since all nodes are assigned the same colour. For the induction step, suppose $b^{t+1}_v = b^{t+1}_w$, for two $0$-simplices $v$ and $w$ in two arbitrary complexes $\gK_1, \gK_2$. As vertices are only upper adjacent and have no boundary simplices, $b^t_v = b^t_w$ and $b_\ua^t(v) = b_\ua^t(w)$. Using the definition of the latter multi-set equality:
$$\ldblbrace b_{z}^{t} \mid (b_{z}^{t}, \cdot) \in b_\ua^{t}(v) \rdblbrace
= \ldblbrace b_{u}^{t} \mid (b_{u}^{t}, \cdot) \in b_\ua^{t}(w) \rdblbrace.$$ 
Equivalently, this can be rewritten in terms of the upper-neighbours of the vertices as
$$\ldblbrace b_{z}^{t} \mid z \in \nup(v) \rdblbrace
= \ldblbrace b_{u}^{t} \mid u \in \nup(w) \rdblbrace.$$ 
By the induction hypothesis, $a^t_v = a^t_w$ and $a_\ua^t(v) = a_\ua^t(w)$. Since these are the arguments the WL hash function uses to compute the colours of $v$ and $w$ at the next step, we obtain $a^{t+1}(v) = a^{t+1}(w)$. 
\end{proof}

Informally, this proof shows that the information coming from the higher-dimensional simplices of the complex will refine the colouring of the vertices. This means that SWL will be able to distinguish just through its vertex-level colour histogram at least the same set of simplicial complexes (and graphs) that WL can distinguish. However, this proof ignores the histograms of the higher-levels and these can indeed be used to show that SWL is strictly more powerful than WL when using a clique complex lifting. This is done in Theorem~\ref{theo:swl_more_powerful_than_wl}. 

\begin{proof}[Proof of Theorem \ref{theo:swl_more_powerful_than_wl}]
Based on Lemma \ref{lemma:wl_is_at_most_swl}, it is sufficient to present a pair of graphs that cannot be distinguished by WL, but whose clique complexes can be distinguished by SWL. Such a pair is included in Figure \ref{fig:cc_expresiveness}. While WL produces the same colouring for both graphs, one clique complex contains two triangles, while the other has no triangles.  
\end{proof}


\begin{proof}[Proof of Lemma~\ref{lemma:mpsns_at_most_as_powerful_as_swl}]
Let $c^t$ and $h^t$ be the colouring at iteration $t$ of SWL and the $t$-th layer of an MPSN, respectively. We consider an MPSN model with $L$ layers. For $t > L$, we assume $h^t = h^L$. We will show by induction that $c^t$ refines the colouring of $h^t$. For this proof, it is convenient to use the most general version of SWL, containing the complete set of adjacencies. 

The base case trivially holds. For the induction step, suppose we have two simplicies $\sigma$ and $\tau$ such that $c^{t+1}(\sigma) = c^{t+1}(\tau)$. Because the SWL colouring is an injective mapping, the arguments to the HASH must also be equal. This means that $c^t_\sigma = c^t_\tau$ and the multi-sets of colours formed by their neighbours are identical: $c^t_\da(\sigma) = c^t_\da(\tau), c^t_\ua(\sigma) = c^t_\ua(\tau), c_\gB^t(\sigma) = c_\gB^t(\tau), c_\gC^t(\sigma) = c_\gC^t(\tau)$. By the induction hypothesis, these multi-sets will also be equal under the colouring $h^t$. Enumerating all, $h^t(\sigma) = h^t(\tau), h^t_\da(\sigma) = h^t_\da(\tau), h^t_\ua(\sigma) = h^t_\ua(\tau), h_\gB^t(\sigma) = h_\gB^t(\tau), h_\gC^t(\sigma) = h_\gC^t(\tau)$. Because the exact same multi-sets are supplied as input to the message, aggregate and update functions, their output will also be the same for $\sigma$ and $\tau$. Thus, $h^{t+1}_\sigma = h^{t+1}_\tau$. 
\end{proof}

\begin{proof}[Proof of Theorem~\ref{thm:mpsn_as_powerful_as_swl}]
By Lemma \ref{lemma:mpsns_at_most_as_powerful_as_swl}, we only need to show that for an MPSN model satisfying the conditions in the theorem, we have that $h^t \sqsubseteq c^t$.

The base case can be proved by definition. For the step case, given that the update, aggregate and message functions are injective, their composition is also injective. Therefore, for any two simplicies $\sigma, \tau$ with $h^{t+1}_\sigma = h^{t+1}_\tau$, the multi-sets of colours in their neighbourhoods are also the same. As in our previous proofs, by applying the induction hypothesis, the inputs to the SWL HASH function at iteration $t$ for $\sigma$ and $\tau$ are also equal and $c^{t+1}_\sigma = c^{t+1}_\tau$. It follows $h^t \sqsubseteq c^t$, $c^t \sqsubseteq h^t$ (Lemma~\ref{lemma:mpsns_at_most_as_powerful_as_swl}) and, consequently, $c^t \equiv h^t$. 
\end{proof} 

\subsection{Higher-Order WL and Strongly Regular Graphs}\label{app:higher_order_wl_and_srgs}

Higher-order variants of the standards WL procedure operate on node tuples rather than single nodes and iteratively apply color refinement steps thereon.

\paragraph{$k$-WL} The $k$-WL is one such higher-order variants. It specifically operates on node $k$-tuples by refining their colors based on the generalized notion of $j$-neighborhood. The $j$-neighborhood ($j \in \{1, \dots, k \}$) for node $k$-tuple $\mathbf{v} = (v_1, v_2, \mathellipsis, v_k)$ is defined as $\gN_j(\mathbf{v}) = \{ (v_1, \mathellipsis, v_{j-1}, w, v_{j+1}, \mathellipsis, v_k) | w \in \gV_G \} $. The algorithm first initialises node tuples based on their isomorphism type: two $k$-tuples $\mathbf{v}^a=(v_1^a, v_2^a, \dots, v_k^a)$, $\mathbf{v}^b=(v_1^b, v_2^b, \dots, v_k^b)$ have the same isomorphism type (and are thus assigned the same initial colour $c_{\mathbf{v}^a} = c_{\mathbf{v}^b}$) iff (i) $\forall \, i,j \in \{1, \dots, k \}, v_i^a=v_j^a \Leftrightarrow v_i^b=v_j^b$, (ii) $\forall \, i,j \in \{1, \dots, k \}, v_i^a \sim v_j^a \Leftrightarrow v_i^b \sim v_j^b$, where $\sim$ indicates adjacency. Given this initial colouring, the procedure iteratively applies the following color refinement step
\begin{align}
    c_\mathbf{v}^{t+1} &= \mathrm{HASH}\Big( c_\mathbf{v}^t, M^t(\mathbf{v}) \Big),\\ 
    M^{t}(\mathbf{v}) &= \bigl(\ldblbrace c_\mathbf{u}^t | \mathbf{u} \in \gN_{j}(\mathbf{v}) \rdblbrace \big| j = 1, 2, \dots, k \bigr)
\end{align}
\noindent until the colouring does not change further.
The $k$-WL procedure can be employed to \emph{test} the isomorphism between graphs in the same way as the standard WL one is. For any $k\geq 2$, it is known that $(k+1)$-WL test is strictly stronger than $k$-WL one, i.e. there exist exemplary pairs of non-isomorphic graphs that $k$-WL cannot distinguish while ($k+1$)-WL can, but not vice-versa. Local variants of $k$-WL have recently been introduced in \citet{morris2020weisfeiler}.

\paragraph{$k$-FWL} The $k$-Folklore WL procedure ($k$-FWL) is another higher-order variant of the standard WL. Similarly to $k$-WL, it operates by refining the colors of node $k$-tuples, initialised based on their isomorphism type. However, it employs a different notion of neighborhood and refinement step. The Folklore $j$-neighborhood for node $k$-tuple $\mathbf{v}$ is defined as $\gN_j^F(\mathbf{v}) = \big( (j, v_2, \mathellipsis, v_k), (v_1, j, \mathellipsis, v_k), \mathellipsis, (v_1, \mathellipsis, v_{k-1}, j) \big) $, with $j \in \gV_G$. The algorithm iteratively applies the steps
\begin{align}
    c_\mathbf{v}^{t+1} &= \mathrm{HASH}\Big( c_\mathbf{v}^t, M^{F,t}(\mathbf{v}) \Big), \label{eq:fwl}\\
    M^{F,t}(\mathbf{v}) &= \ldblbrace \big( c_\mathbf{u}^t | \mathbf{u} \in \gN_{j}^F(\mathbf{v}) \big) \big| j \in \gV_G \rdblbrace
\end{align}
\noindent until the colouring does not change anymore.
It is known that $k$-FWL is equivalent to ($k+1$)-WL for $k \geq 2$.

\paragraph{Strongly Regular Graphs}
A Strongly Regular graph in the family {\em SR($n$,$d$,$\lambda$,$\mu$)} is a regular graph with $n$ nodes and degree $d$, for which every two adjacent nodes always have $\lambda$ mutual neighbours and every two non-adjacent nodes always have $\mu$ mutual neighbours. This class of graphs is of particular interest due to the following lemma.
\begin{lemma}\label{lemma:sr_2-fwl}
No pair of Strongly Regular graphs in family {\em SR($n$,$d$,$\lambda$,$\mu$)} can be distinguished by the $2$-FWL test.
\end{lemma}
\begin{proof}
    
    Let us denote by $\gV_G^2$ the set of all node $2$-tuples in graph $G$.
    We note that three isomorphism types are induced by considering node $2$-tuples:
    \begin{enumerate}
        \item[(1)] \emph{node type}: $\mathbf{v}=(v_1, v_1)$
        \item[(2)] \emph{edge type}: $\mathbf{v}=(v_1, v_2)$ with $v_1 \sim v_2$ (the two nodes are adjacent in the original graph)
        \item[(3)] \emph{non-edge type}: $\mathbf{v}=(v_1, v_2)$ with $v_1 \not\sim v_2$ (the two nodes are \emph{not} adjacent in the original graph).
    \end{enumerate}
    These three isomorphism types partition the tuple set $\gV_G^2$ into the three subsets ${\gV_G^2}_{(1)}, {\gV_G^2}_{(2)}, {\gV_G^2}_{(3)}$ such that any tuple $\mathbf{v} \in {\gV_G^2}_{(i)}$ is of isomorphism type $i$. We write $\mathbf{v}^{(i)}$ to indicate $\mathbf{v} \in {\gV_G^2}_{(i)}$ for simplicity.
    
    At initialisation, the $2$-FWL algorithm assigns a colour to each tuple $\mathbf{v} \in \gV_G^2$ based on its isomorphism type, that is $\forall \mathbf{v} \in {\gV_G^2}_{(i)}, c_\mathbf{v} = c_i^0$. The colouring is therefore constant within partitions.
    Then, we notice that the colouring is kept constant within partitions through the application of the refinement steps described by Equation~\ref{eq:fwl}. In other words, the $2$-FWL procedure cannot produce a colour partitioning of the set of node $2$-tuples that is finer than the one at initialisation.
    This is shown by induction on the step $t$ of color refinement.
    
    The base case evidently holds for $t=0$ since all tuples in the same partition are assigned the same colour by the $2$-FWL initialisation procedure.
    
    For the induction step we assume that the colouring is constant within each partition at $t$ and show that it maintains constant at $t+1$, that is, after the application of one colour refinement step. This is proved by showing that all node tuples within the same partition have their colour refined identically. We will show this for each of the three partitions separately, leveraging on the induction hypothesis and the properties of Strongly Regular graphs.
    
    A node tuple $\mathbf{v} = (v_1, v_1) \in {\gV_G^2}_{(1)}$ has $\gN_j^F(\mathbf{v}) = \big((j, v_1), (v_1, j)\big), j \in \gV_G$. Therefore, \emph{any} $\mathbf{v} \in {\gV_G^2}_{(1)}$ has exactly:
    \begin{itemize}
        \item ($j = v_1$) $1$ neighborhood of the form $\big( \mathbf{w}^{(1)}, \mathbf{w}^{(1)} \big)$, associated with color tuple $(c_1^t, c_1^t)$;
        \item ($j \sim v_1$) $d$ neighborhoods of the form $\big( \mathbf{w}^{(2)}, \mathbf{u}^{(2)} \big)$, associated with color tuple $(c_2^t, c_2^t)$;
        \item ($j \not\sim v_1$) $n-d-1$ neighborhoods of the form $\big( \mathbf{w}^{(3)}, \mathbf{u}^{(3)} \big)$, associated with color tuple $(c_3^t, c_3^t)$.
    \end{itemize}
    For \emph{any} $\mathbf{v} \in {\gV_G^2}_{(1)}$ we thus have
    \begin{align*}
    c_\mathbf{v}^{t+1} &= \text{HASH} \Big(c_{1}^{t}, M^{F,t}(\mathbf{v})\Big) \\
    M^{F,t}(\mathbf{v}) &= \ldblbrace 
            \underbrace{(c_1^t, c_1^t)}_{\text{$1$ time}},
            \underbrace{(c_2^t, c_2^t)}_{\text{$d$ times}},
            \underbrace{(c_3^t, c_3^t)}_{\text{$n-d-1$ times}}
    \rdblbrace.
    \end{align*}
    
    A node tuple $\mathbf{v} = (v_1, v_2) \in {\gV_G^2}_{(2)}$ has $\gN_j^F(\mathbf{v}) = \big((j, v_2), (v_1, j)\big), j \in \gV_G$. Therefore, \emph{any} $\mathbf{v} \in {\gV_G^2}_{(2)}$ has exactly:
    \begin{itemize}
        \item ($j = v_2$) $1$ neighborhood of the form $\big( \mathbf{w}^{(1)}, \mathbf{u}^{(2)} \big)$, associated with color tuple $(c_1^t, c_2^t)$;
        \item ($j = v_1$) $1$ neighborhood of the form $\big( \mathbf{w}^{(2)}, \mathbf{u}^{(1)} \big)$, associated with color tuple $(c_2^t, c_1^t)$;
        \item ($j \sim v_2, j \sim v_1$) $\lambda$ neighborhoods of the form $\big( \mathbf{w}^{(2)}, \mathbf{u}^{(2)} \big)$, associated with color tuple $(c_2^t, c_2^t)$;
        \item ($j \sim v_2, j \not\sim v_1$) $d-\lambda$ neighborhoods of the form $\big( \mathbf{w}^{(2)}, \mathbf{u}^{(3)} \big)$, associated with color tuple $(c_2^t, c_3^t)$;
        \item ($j \not\sim v_2, j \sim v_1$) $d-\lambda$ neighborhoods of the form $\big( \mathbf{w}^{(3)}, \mathbf{u}^{(2)} \big)$, associated with color tuple $(c_3^t, c_2^t)$;
        \item ($j \not\sim v_2, j \not\sim v_1$) $k = n-2-2d+\lambda$ neighborhoods of the form $\big( \mathbf{w}^{(3)}, \mathbf{u}^{(3)} \big)$, associated with color tuple $(c_3^t, c_3^t)$.
    \end{itemize}
    For \emph{any} $\mathbf{v} \in {\gV_G^2}_{(2)}$ we have
    \begin{align*}
    c_\mathbf{v}^{t+1} &= \text{HASH} \Big(c_{2}^{t}, M^{F,t}(\mathbf{v})\Big) \\
    M^{F,t}(\mathbf{v}) &= \ldblbrace 
            \underbrace{(c_1^t, c_2^t)}_{\text{1 time}},
            \underbrace{(c_2^t, c_1^t)}_{\text{1 time}},
            \underbrace{(c_2^t, c_2^t)}_{\text{$\lambda$ times}}, \\
            &\quad\qquad\underbrace{(c_2^t, c_3^t)}_{\text{$d-\lambda$ times}},
            \underbrace{(c_3^t, c_2^t)}_{\text{$d-\lambda$ times}},
            \underbrace{(c_3^t, c_3^t))}_{\text{$k$ times}}
    \rdblbrace.
    \end{align*}
    
    A node tuple $\mathbf{v} = (v_1, v_2) \in {\gV_G^2}_{(3)}$ has $\gN_j^F(\mathbf{v}) = \big((j, v_2), (v_1, j)\big), j \in \gV_G$. Therefore, \emph{any} $\mathbf{v} \in {\gV_G^2}_{(3)}$ has exactly:
    \begin{itemize}
        \item ($j = v_2$) $1$ neighborhood of the form $\big( \mathbf{w}^{(1)}, \mathbf{u}^{(3)} \big)$, associated with color tuple $(c_1^t, c_3^t)$;
        \item ($j = v_1$) $1$ neighborhood of the form $\big( \mathbf{w}^{(3)}, \mathbf{u}^{(1)} \big)$, associated with color tuple $(c_3^t, c_1^t)$;
        \item ($j \sim v_2, j \sim v_1$) $\mu$ neighborhoods of the form $\big( \mathbf{w}^{(2)}, \mathbf{u}^{(2)} \big)$, associated with color tuple $(c_2^t, c_2^t)$;
        \item ($j \sim v_2, j \not\sim v_1$) $d-\mu$ neighborhoods of the form $\big( \mathbf{w}^{(2)}, \mathbf{u}^{(3)} \big)$, associated with color tuple $(c_2^t, c_3^t)$;
        \item ($j \not\sim v_2, j \sim v_1$) $d-\mu$ neighborhoods of the form $\big( \mathbf{w}^{(3)}, \mathbf{u}^{(2)} \big)$, associated with color tuple $(c_3^t, c_2^t)$;
        \item ($j \not\sim v_2, j \not\sim v_1$) $k = n-2-2d+\mu$ neighborhoods of the form $\big( \mathbf{w}^{(3)}, \mathbf{u}^{(3)} \big)$, associated with color tuple $(c_3^t, c_3^t)$.
    \end{itemize}
    For \emph{any} $\mathbf{v} \in {\gV_G^2}_{(3)}$, we then obtain
    \begin{align*}
    c_\mathbf{v}^{t+1} &= \text{HASH} \Big(c_{3}^{t}, M^{F,t}(\mathbf{v})\Big) \\
    M^{F,t}(\mathbf{v}) &= \ldblbrace 
            \underbrace{(c_1^t, c_3^t)}_{\text{1 time}},
            \underbrace{(c_3^t, c_1^t)}_{\text{1 time}},
            \underbrace{(c_2^t, c_2^t)}_{\text{$\mu$ times}}, \\
            &\quad\qquad\underbrace{(c_2^t, c_3^t)}_{\text{$d-\mu$ times}},
            \underbrace{(c_3^t, c_2^t)}_{\text{$d-\mu$ times}},
            \underbrace{(c_3^t, c_3^t))}_{\text{$k$ times}}
    \rdblbrace.
    \end{align*}
    
    This proves the induction and confirms that all tuples in the same partition have the same colour at any colour refinement time step $t$.
    
    If the colouring is constant within partitions at any $2$-FWL step, then the colour histogram associated with a graph at step $t$ purely depends on the cardinality of each of the three. We notice that, for \emph{any} $G \in$ \emph{SR($n$,$d$,$\lambda$,$\mu$)}, they are completely determined by the first two parameters with
    \begin{itemize}
        \item $|{\gV_G^2}_{(1)}| = n$
        \item $|{\gV_G^2}_{(2)}| = n d$
        \item $|{\gV_G^2}_{(3)}| = |\gV_G^2| - (n + n d)$.
    \end{itemize}

    Given the above, any two graphs $G_1, G_2 \in$ \emph{SR($n$,$d$,$\lambda$,$\mu$)} are associated with the same colour histograms at any step of the $2$-FWL procedure and, therefore, cannot possibly deemed non-isomorphic by the last.
\end{proof}

We leverage on Lemma~\ref{lemma:sr_2-fwl} to prove Theorem~\ref{thm:swl_noless_than_3wl}.
\begin{proof}[Proof of Theorem~\ref{thm:swl_noless_than_3wl}]
    In virtue of Lemma~\ref{lemma:sr_2-fwl} and the fact that $2$-FWL is as powerful as $3$-WL, Theorem~\ref{thm:swl_noless_than_3wl} is proved by exhibiting a pair of Strongly Regular graphs in the same family that are distinguished by the SWL test. This pair is given by the two graphs in Figure~\ref{fig:SR}: Rook’s $4$x$4$ graph ($G_1$) and the Shrikhande graph ($G_2$), (the only) members of the \emph{SR(16,6,2,2)} family of Strongly Regular graphs. The SWL test which considers their clique $3$-complexes distinguish them due to the fact that, differently from $G_1$, $G_2$ possesses no $4$-cliques, thus its associated complex has no $3$-simplices.
\end{proof}

\section{Proofs of Linear Regions Results}
\label{app:linear regions}

\paragraph{Background on Hyperplane Arrangements}
A function $f\colon \mathbb{R}^N\to\mathbb{R}^M$ is a \emph{piecewise linear function} if its graph $\{(x,f(x))\colon x\in \mathbb{R}^N\}\subseteq\mathbb{R}^N\times \mathbb{R}^M$ consists of a finite number of polyhedral pieces. Projecting these polyhedra back onto $\mathbb{R}^N$ by $(x,y)\mapsto x$  defines a polyhedral subdivision of $\mathbb{R}^N$. 
The \emph{linear regions} of the function are the $N$-dimensional pieces in this subdivision. 
These are the (inclusion maximal) connected regions of the input space where the function is affine linear. 

Let $\psi\colon \mathbb{R}\to\mathbb{R}$; $s\mapsto \max\{0,s\}$ be the linear rectification. 
A ReLU with $N$ inputs defines a function $y\colon x\mapsto \psi(w^\top x)$, which for any fixed value of the weight vector $w\in\mathbb{R}^N$, $w\neq0$, has gradient with respect to the input vector $x\in\mathbb{R}^N$ equal to $0$ on the open halfspace $\{x\colon w^\top x< 0\}$ and equal to $w$ on the open halfspace $\{x\colon w^\top  x> 0\}$. Hence a ReLU defines a piecewise linear function with two linear regions. 
A layer of ReLUs $\psi(w_i^\top x)$, $i=1,\ldots,M$ has linear regions given by the intersection of linear regions of the individual ReLUs. 
The number of linear regions of the function represented by the layer is equal to the number of connected components that are left behind once we remove $\cup_{i=1}^M A_i$ from $\mathbb{R}^N$, where $A_i=\{x\in\mathbb{R}^N\colon w_i^\top x=0\}$ is the hyperplane dividing the two linear regions of the $i$th ReLU. 
Hence the linear regions of a layer of ReLUs can be described in terms of a \emph{hyperplane arrangement}, i.e. a collection $\mathcal{A}=\{A_i\colon i=1,\ldots, M\}$ of hyperplanes. 

An arrangement of hyperplanes in $\mathbb{R}^N$ is \emph{in general position} if the intersection of any $k$ hyperplanes in the arrangement has the expected co-dimension, $k$. 
We will focus on \emph{central} arrangements, where each hyperplane contains the origin. A central arrangement is in general position when the normal vectors $w_{i_1},\ldots, w_{i_k}$ of any $k\leq N$ of the hyperplanes are linearly independent. 
The following well-known result from the theory of hyperplane arrangements will be particularly important in our discussion. 
\begin{theorem}
\label{thm:zaslavsky-central-genpos}
Let $\mathcal{A}$ be a central arrangement of $M$ hyperplanes in $\mathbb{R}^N$ in general position. 
Then the number of regions of the arrangement, denoted $r(\mathcal{A})$, is equal to $2\sum_{j=0}^{N-1}{M-1\choose j}$. 
This is also the maximum number of regions of any central arrangement of $M$ hyperplanes in $\mathbb{R}^N$. 
\end{theorem}
This result can be derived from Zaslavsky's theorem \citep{zaslavsky1975facing}, which expresses the number of regions of an arbitrary arrangement, not necessarily in general position, in terms of properties of a partially ordered set, namely the collection of intersections of the hyperplanes in the arrangement partially ordered by reverse inclusion. 

We will focus on central arrangements, but we point out that similar results to Theorem~\ref{thm:zaslavsky-central-genpos} can be derived for the case of non-central hyperplane arrangements. An non-central arrangement of (affine) hyperplanes in $\mathbb{R}^N$ is in general position when any $k\leq N$ of the hyperplanes intersect in a set of dimension $N-k$, and any $k>N$ of the hyperplanes have an empty intersection. For such an arrangement, the number of regions is $\sum_{j=0}^N{M\choose j}$. 

The main challenges in computing the number of regions defined hyperplane arrangements happen when the hyperplanes satisfy some type of constraints and are not in general position. 
The type of layers that we discuss in the following correspond to central arrangements subject to certain constraints, namely that the normal vectors are the rows of a matrix with a particular block Kronecker product structure. 

\subsection{GNNs} 
\begin{proof}[Proof of Theorem~\ref{thm:linear regions of GNNs}]
    For simplicity, we write $Y=\mathcal{H}(A,H)^{T}\in \R^{d\times S_0}$ and $V=W^{T}\in \R^{m\times d}$. We denote $Y_{:j}$ the $j$the column of $Y$, and
    $V_{i:}$ the $i$th row of $V$. The GNN layer defines hyperplanes, for $i=1,\ldots, m,\; j=1,\ldots, S_0$,
    \begin{equation*}
        A_{ij} := \bigl\{Y\in \R^{d\times S_0}: V_{i:} Y_{:j}=0\bigr\}. 
    \end{equation*}
    The arrangement $\mathcal{A}=\{A_{ij}\colon i=1,\ldots, m, j=1,\ldots, S_0\}$ is a direct sum of the arrangements $\mathcal{A}_j=\{A_{ij}\colon i=1,\ldots, m\}$, $j=1,\ldots,S_0$. 
    It can be shown \citep[see][]{zaslavsky1975facing} that this implies $r(\mathcal{A})=\prod_{j=1}^{S_0} r(\mathcal{A}_j)$.  
    
    Each $\mathcal{A}_j$ is an arrangement of $m$ hyperplanes in $\mathbb{R}^N$, $N=d S_0$, whose normals span a subspace of dimension at most $d$, irrespective of $S_0$. 
    Counting the number of regions defined by $\mathcal{A}_j$ is equivalent to counting the number of regions defined by its \emph{essentialization} $\operatorname{ess}(\mathcal{A}_j)$, which is the arrangement that the hyperplanes define on the span of their normal vectors. 
    We can regard $\operatorname{ess}(\mathcal{A}_j)$ as a (central) arrangement of $m$ hyperplanes in $\mathbb{R}^d$ with normals $V_{i:}\in\mathbb{R}^d$, $i=1,\ldots,m$. 
    For generic choices of the weight matrix $W^\top=V$, this is a central arrangement in general position. 
    Hence, by Theorem~\ref{thm:zaslavsky-central-genpos},  $r(\mathcal{A}_j) = r(\operatorname{ess}(\mathcal{A}_j))= 2\sum_{i=0}^{d-1}{m-1\choose i}$. 
    
    For the number of regions of the entire arrangement $\mathcal{A}$, corresponding to the number of linear regions of the function expressed by the layer, 
    we obtain 
    $R_{\rm GNN} = r(\mathcal{A}) = \prod_{j=1}^{S_0}r(\mathcal{A}_j) =  \left(2\sum_{i=0}^{d-1}{m-1\choose i}\right)^{S_0}$. 
\end{proof}

\subsection{SCNNs}

\begin{proof}[Proof of Theorem~\ref{thm:linear regions of SCNN}] 
By the definition of the SCNN layer, for each dimension $n$, 
each of the $S_n$ $n$-simplices in the simplicial complex has $d_n$ input features. 
Similar to the proof of Theorem~\ref{thm:linear regions of GNNs}, the arrangement for the $n$-dimensional simplices corresponds to a direct sum of $S_n$ arrangements, each of $m_n$ hyperplanes in $\R^{d_n}$. 
Hence the number of linear regions for this part of the complex is 
\begin{equation}\label{eq:linear regions n-complex}
    \left(2\sum_{i=0}^{d_n-1}{m_n-1\choose i}\right)^{S_n}.
\end{equation}
Now for the entire layer, the arrangements for the different $n$ are also combined as a direct sum, so that their number of regions multiply. 
We have $n$ ranging from dimension $0$ to $p$, 
so that 
\begin{equation}
    \prod_{n=0}^p\left(2\sum_{i=0}^{d_n-1}{m_n-1\choose i}\right)^{S_n}.
\end{equation}
This concludes the proof. 
\end{proof}

\subsection{MPSNs}

We can rewrite \eqref{eq:m} more generality and more concisely as follows. For each $n$ the output features on $\mathcal{S}_n$ can be written as 
\begin{align}
H_n^{\rm out} = & \psi(M_n H_n W_n + U_{n}H_{n-1} W_{n-1} + O_nH_{n+1} W_{n+1}) \notag\\
=&\psi\left([U_n H_{n-1}| M_n H_n | O_n H_{n+1} ]\left[\begin{smallmatrix}W_{n-1}\\W_{n}\\W_{n+1}\end{smallmatrix}\right]\right), 
\label{eq:mpsn n}
\end{align}
for some fixed matrices $U_n\in\mathbb{R}^{S_n\times S_{n-1}}$, $M_n\in\mathbb{R}^{S_n\times S_n}$ and $O_{n}\in\mathbb{R}^{S_n\times S_{n+1}}$ depending only on the simplicial complex. 
To avoid clutter, we omit the superscript ``in'' of the input feature matrices $H_n$. 
Concatenating \eqref{eq:mpsn n} for all $n$, we can write the entire MPSN layer as 
\begin{multline}
\left[\begin{smallmatrix}
H_0^{\rm out}\\H_1^{\rm out}\\H_2^{\rm out}\\\vdots\\H_p^{\rm out}
\end{smallmatrix}
\right] = 
\psi\left(
\left[\begin{smallmatrix}
M_0H_0&O_0H_1&      && \\
U_1H_0&M_1H_1&O_1H_{2}&& \\
    &U_2H_1&M_2H_2&O_2H_{3}&\\
    &   &     &\ddots&
\end{smallmatrix}\right]
\left[\begin{smallmatrix}
W_0\\W_1\\W_2\\\vdots\\W_p
\end{smallmatrix}
\right]
\right). 
\label{eq:repr1}
\end{multline}
We will use this representation (or rather \ref{eq:mpsn n}) in the proof of the first bound in Theorem~\ref{thm:linear regions of MPSN}. 

It is also useful to write the linear function in standard form. Using Roth's lemma, we can write \eqref{eq:mpsn n} as 
\begin{align}
\operatorname{vec}(H_n^{\rm out}) =& 
\psi\Bigl(\bigl[ W_{n-1}^\top \otimes U_n |W_n^\top\otimes M_n|W_{n+1}^\top\otimes O_n\bigr]\notag\\
&\quad\qquad\times\operatorname{vec}\bigl(\bigl[H_{n-1} |H_{n}| H_{n+1}\bigr]\bigr)\Bigr). 
\label{eq:MPSN n standard}
\end{align}
Now, concatenating over $n$ yields the expression $\psi(W H)$ from \eqref{eq:MPSN full standard} for the entire layer, with the matrix $W\in\mathbb{R}^{M\times N}$ from \eqref{eq:MPSNdimkron}.

\begin{proof}[Proof of Proposition~\ref{prop:simplecase}]
This result is known in theory of partial orders and hyperplane arrangements. We include a proof which illustrates the evaluation of the characteristic polynomial. 
If there is $K$ so that $\operatorname{rank}(W_{B:})=\min\{|B|,K\}$ for all $B$, then Lemma~\ref{lem:whitney} can be evaluated as 
\begin{align*}
r(\mathcal{A}) = & \sum_{B}(-1)^{|B|-\min\{|B|,K\}}\\
=& \sum_{j=0}^K\sum_{B\in{\{1,\ldots, M\} \choose j}}1 +\sum_{j=K+1}^M\sum_{B\in {\{1,\ldots, M\}\choose j}}(-1)^{j-K} \\
=& \sum_{j=0}^K {M\choose j} +  (-1)^{M-K}\sum_{j=0}^{M-(K+1)} {M\choose j} (-1)^{j} \\
=&  \sum_{j=0}^K {M\choose j} + (-1) {M-1\choose K}\\
=&2\sum_{j=0}^{K-1}{M-1\choose j},
\end{align*}
which is what was claimed. 
\end{proof}

We now proceed with the proof of Theorem~\ref{thm:linear regions of MPSN}. 
We will use the following. 
\begin{proposition}
\label{prop:rank}
If $W, W'\in\mathbb{R}^{M\times N}$ are two matrices with $\rank(W_{B:})\geq\rank(W'_{B:})$ for all $B\subseteq \{1,\ldots, M\}$, then $r(\mathcal{A})\geq r(\mathcal{A}')$. 
\end{proposition}

\begin{proof}[Proof of Proposition~\ref{prop:rank}]
Notice that if $W'_{B:}$ is not full rank, then $W'$ solves a polynomial system. More precisely, some minors (determinants of sub-matrices) of $W'_{B:}$ vanish. 
Hence, increasing the rank corresponds to stepping outside of the solution set to a polynomial system. 
This can be accomplished by an arbitrarily small perturbation of the matrix. 
On the other hand, the number of regions of a central arrangement with normals $W$ corresponds to the number of vertices of a polytope which is the convex hull of points parametrized by the entries of $W$. 
The number of vertices is a lower semi-continuous function of the considered polytope \citep[see][Section~5.3]{Gruenbaum2003}, and hence the number of regions is a lower semi-continuous function of the entries of $W$. This means that for sufficiently small perturbations of the entries, the number of regions of the corresponding hyperplane arrangement does not decrease. 
\end{proof}

Further, we will use an inequality for the rank of a block matrix \citep[see, e.g.][Theorem~19]{doi:10.1080/03081087408817070}. 
\begin{lemma}
\label{lem:block-matrix-rank}
Let $A\in\mathbb{R}^{k\times l}$, $B\in\mathbb{R}^{m\times l}$, $C\in\mathbb{R}^{m\times n}$ be matrices. Then 
$\rank(\left[\begin{smallmatrix} A & 0\\B & C \end{smallmatrix}\right]) \geq \rank(A) + \rank(C)$. 
\end{lemma} 

\begin{proof}[Proof of Theorem~\ref{thm:linear regions of MPSN}] The proof of the first upper bound is analogous to Theorem~\ref{thm:linear regions of SCNN}. 
The difference is now we consider also the boundary and co-boundary simplices that interact with an $n$-simplex in the MPSN. 
We use the expression \eqref{eq:mpsn n}. 
The difference compared with Theorem~\ref{thm:linear regions of SCNN} lies in the number of input features for each $n$, which here results in 
\begin{equation}\label{eq:number of linear regions, MPSN}
    \prod_{n=0}^p\left(2\sum_{i=0}^{d_{n-1}+d_n+d_{n+1}-1}{m-1\choose i}\right)^{S_n},
\end{equation}
which is the first upper bound. 
%
The second upper bound is the trivial upper bound, which is the maximum possible number of regions of a central arrangement of $M$ hyperplanes in $\mathbb{R}^N$ from Theorem~\ref{thm:zaslavsky-central-genpos}. 
%

The lower bound follows from Proposition~\ref{prop:rank}. 
We verify that the hypothesis of the proposition is satisfied. 
Note that the matrix $W$ from \eqref{eq:MPSNdimkron} is 
lower block triangular, of the form 
$$
W=\left[\begin{array}{cccccc}
\cline{1-2}
\multicolumn{1}{|c}{\ast}& \multicolumn{1}{c|}{\ast}  & 0 & 0 & 0 \\
\cline{1-3}
\ast     &\ast  & \multicolumn{1}{|c|}{\ast} & 0& 0\\
\cline{3-4}
\ast & \ast & \ast & \multicolumn{1}{|c|}{\ast} & 0 \\
\cline{4-5}
\ast & \ast & \ast &\ast& \multicolumn{1}{|c|}{\ast} \\
\ast&\ast&\ast&\ast& \multicolumn{1}{|c|}{\ast}\\
\cline{5-5}     
\end{array}\right]. 
$$
Applying Lemma~\ref{lem:block-matrix-rank} recursively, 
for the network with outputs $H_0^{\rm out}, \ldots, H_{p-1}^{\rm out}$ we find that  $\operatorname{rank}(W)
\geq 
\operatorname{rank}([W_0^\top \otimes M_0 | W_1^\top\otimes O_0 ])
+\sum_{n=1}^{p-1} \operatorname{rank}(W^\top_{n+1}\otimes O_{n})
$. 
A similar expression holds for any selection of rows, $W_{B:}$. 

On the other hand, the corresponding matrix $W'$ for an SCNN is block diagonal with blocks $W_n^\top \otimes M_n$ and hence 
$\operatorname{rank}(W')
= \sum_{n=0}^{p-1} \operatorname{rank}(W^\top_{n}\otimes M_{n})$. 
A similar expression holds for any selection of rows, $W'_{B:}$. 

%
Hence, if $d_{n+1} \geq d_n$ and $\rank((O_n)_{C:})\geq \rank((M_n)_{C:})$ for any subsets $C$ of rows, then, choosing full rank matrices $W_n$, we have that the overall matrix satisfies 
$\rank(W_{B:}) \geq \rank(W'_{B:})$ for any subset $B$ of rows. 
Hence, applying Proposition~\ref{prop:rank} gives the desired result. 
\end{proof}

It is not difficult to obtain case by case improvements of the bounds in Theorem~\ref{thm:linear regions of MPSN} by conducting a more careful analysis of the row independencies in matrix $W$ for specific values of the input feature dimensions $d_0,\ldots,d_p$, output feature dimension $m$, numbers of simplices $S_0,\ldots, S_p$, and the structure of the matrices $U_n,M_n,O_n$, $n=0,\ldots,p$. 

\paragraph{MPSN with Populated Higher-Features}
Finally, we consider the populated higher features for a situation when we are given a simplicial complex but only vertex features are available. 
%
This strategy can increase the network complexity, as proved by Proposition~\ref{prop:ghost} below.

\begin{proposition}[MPSN with populated higher-features] 
\label{prop:ghost}
Consider an arbitrary simplicial complex and an MPSN layer mapping $\mathbb{R}^{S_0\times d_0} \to \mathbb{R}^{S_0\times m}; H_0^{\rm in}\mapsto H_{0}^{\rm out}$, 
whereby higher-dimensional input features
are populated as linear functions of the input vertex features.  
%
Consider further the corresponding SCNN layer which computes $H_0^{\rm out}=\psi(L_0 H_0^{\rm in} W_0)$. 
Then, $R_{\rm MPSN} 
\geq R_{\rm SCNN}$. Furthermore, for certain simplicial complexes and feature dimensions $d_0,\ldots, d_p,m$, the inequality is strict. 
\end{proposition}
The regions for the two cases are illustrated in Figure~\ref{fig:regions_pop}. It shows MPSN (Right) has more regions than GNN/SCNN (Left) and has a higher complexity for populated case.

\begin{proof}[Proof of Proposition~\ref{prop:ghost}]
Focusing on the $\mathcal{S}_0$ output features, the relevant matrices are $[W_0^\top \otimes M_0]$ for the simplicial network and $[W_0^\top \otimes M_0 | W_1^\top \otimes O_0] C$ for the MPSN with populated higher-dimensional features. 
Both matrices have format $m S_0\times d S_0$ and rank $\min\{m,d_0\} \rank(M_0)$. 
However, subsets of rows have different ranks in both cases. 
For illustration, if $d=1$ and $l$ is the smallest number of nonzero entries of any row in $M_0$, then $[W_0^\top M_0]$ has an $m$ row submatrix of rank $\min\{m,l\}$. 
In contrast, an $m$ row submatrix of the augmented matrix will have rank $\min\{m,l+l'\}$, where $l'$ is the smallest number of nonzero entries of a row in $O_0$.
\end{proof}

\begin{figure}[t]
\centering
\begin{tabular}{cc}
GNN / SCNN & MPSN \\ 
\!\!\!
\includegraphics[clip=true,trim=5cm 8.2cm 4.5cm 7.5cm,width=.3\columnwidth]{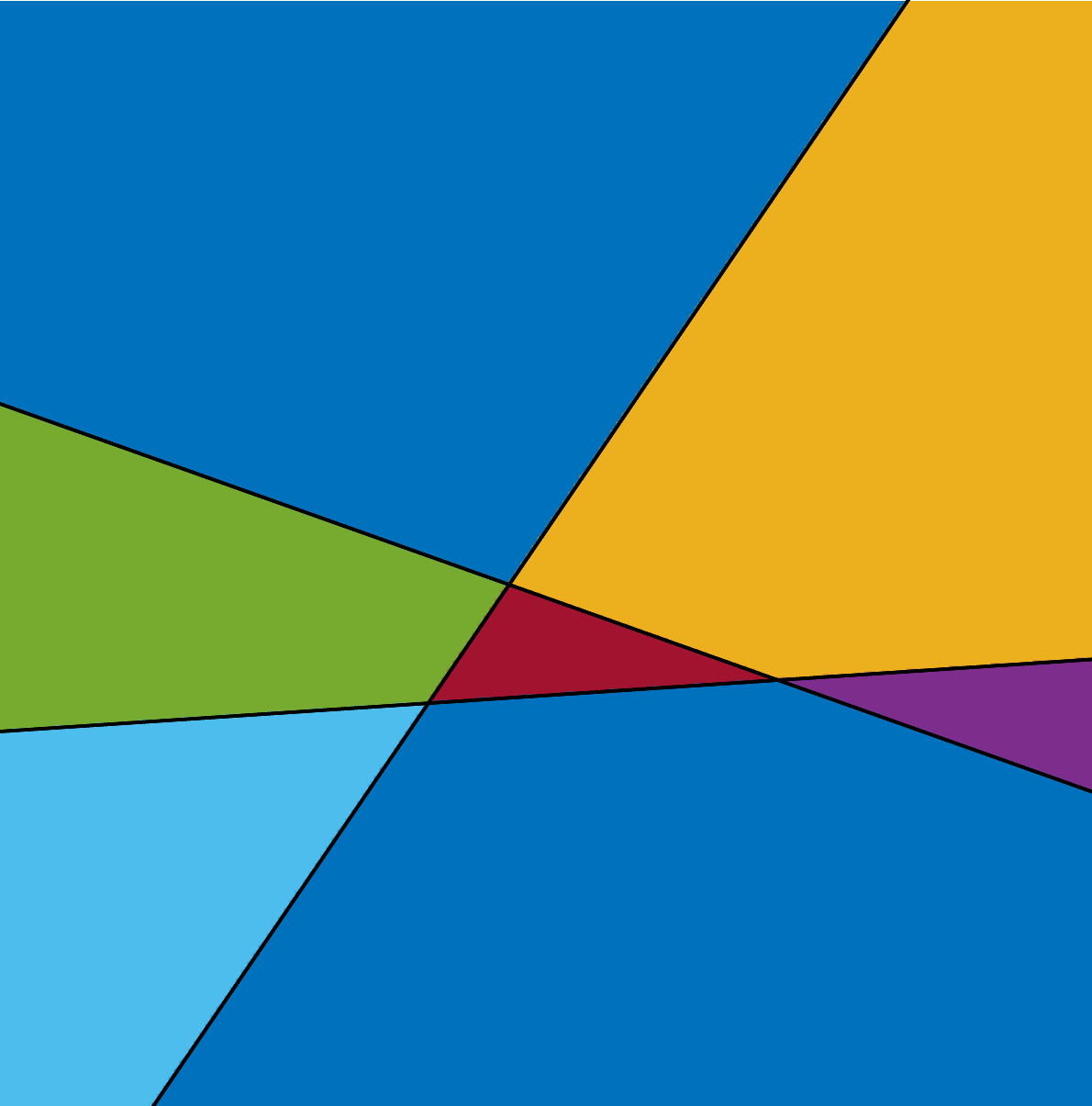}
&\!\!\! 
\includegraphics[clip=true,trim=5cm 8.2cm 4.5cm 7.5cm,width=.3\columnwidth]{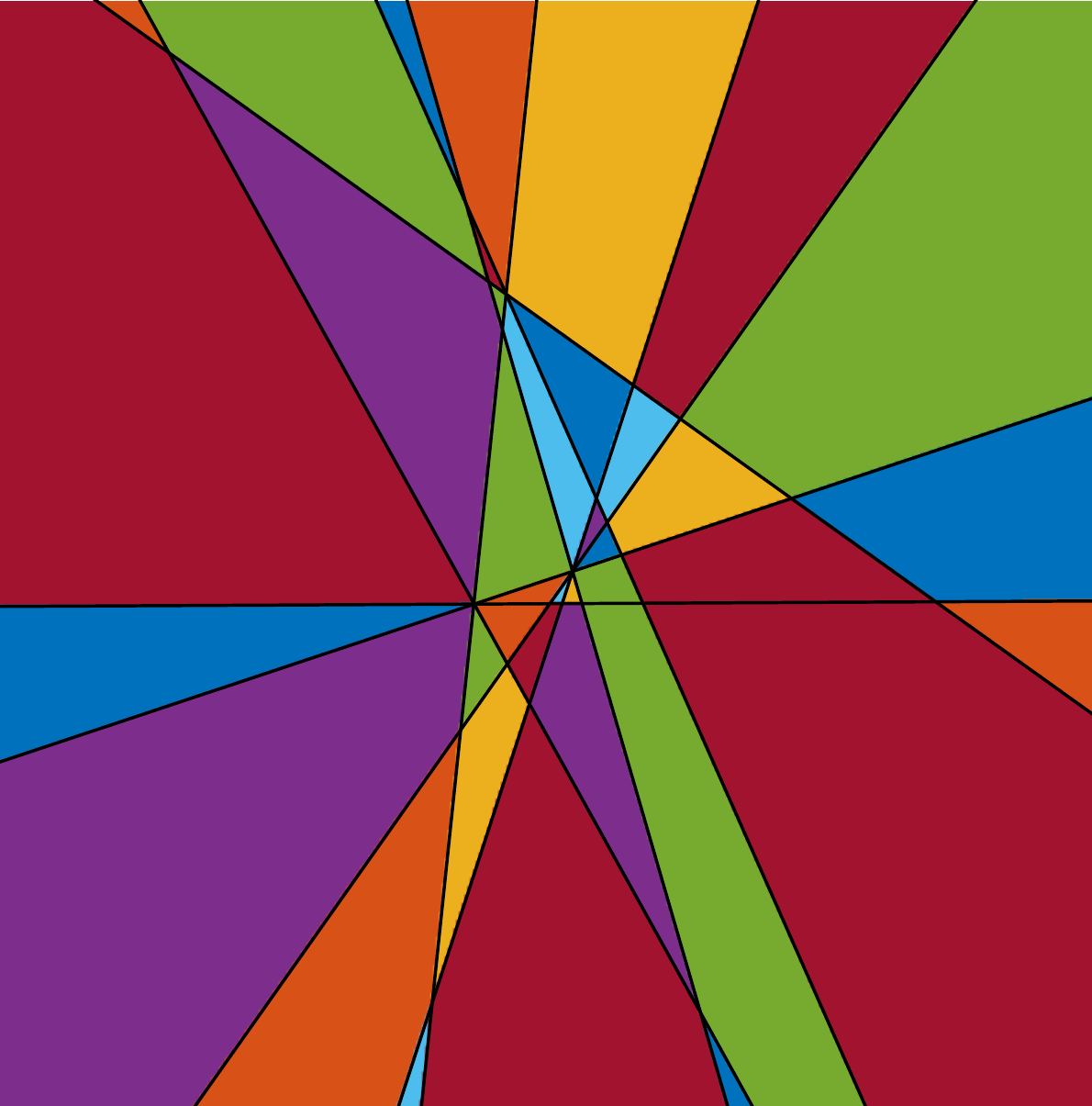} 
\end{tabular}
\vspace{-3pt}
\caption{Shown is a 2D slice of the input vertex feature space of SCNN and MPSN layers with $S_0=S_1=3$, $d_0=d_1=1$, $m=3$, computing output vertex features, with input edge features 
set as a random linear function of the input vertex features.}
    \label{fig:regions_pop}
    \vspace{-1mm}
\end{figure}

\section{Relationship to Convolutions}
\label{app:convolutions}

\paragraph{Background on Hodge Laplacian} The simplicial convolutions described in this section rely on the Hodge Laplacians of the simplicial complex. Such a Laplacian operator $L_p$ is associated with each dimension $p$ of the complex. The operator $L_p$ has a special structure that depends on the boundary operators of the corresponding dimensions $B_p$ and $B_{p+1}$. These matrices are the discrete equivalent of the boundary operators encountered in algebraic topology (see \citet{schaub2020random} for more details). They essentially describe things such as the fact that boundary of a $2$-simplex is formed by its edges with some relative orientation (positive or negative). The Laplacian can be written as a function of these boundary matrices as $L_p = B_p^\top B_p + B_{p+1} B_{p+1}^\top$. We denote the first term by $L_p^{\da}$ and the second by $L_p^{\ua}$, because they encode up and down adjacencies between the $p$-simplicies of the complex and the relative orientation between them. Additionally, let us recall the important relation $B_{p} B_{p+1} = \mathbf{0}$ (i.e. the zero map). This equation, which is fundamental in homology theory and differential geometry, formally specifies that the boundary of a simplex has no boundary (e.g. the boundary of a $2$-simplex has no boundary because it forms a loop and, therefore, it has no endpoints).   

We emphasise that the boundary matrices and the Hodge Laplacian depend on an arbitrary choice of orientation for the simplicial complex. Although these linear operators are orientation-equivariant, more general convolutional layers based on these operators are not guaranteed to be orientation-equivariant. When originally introduced, the orientation equivariance properties of the convolutional operators from \citet{ebli2020simplicial} and \citet{bunch2020simplicial} were not considered. Therefore, when the orientation of the input complex is changed, these models could produce completely different outputs. This is also emphasised by the MPSN ReLU model in Section \ref{sec:experiments}. In Appendix~\ref{app:equiv_inv} we analyse these aspects in more detail.

\paragraph{Simplicial Neural Networks}
\citet{ebli2020simplicial} presented Simplicial Neural Networks (SNNs), neural network models extending convolutional layers to attributed simplicial complexes. The theoretical construction closely resembles the one by \citet{defferrard2016convolutional}, where the standard graph Laplacian is simply replaced by the more general Hodge $p$-Laplacian $L_p$, i.e.\ the generalization of the Laplacian operator to simplices of order $p$. The $p$-th order SNN convolutional layer is defined as 
\begin{equation}\label{eq:SNN}
   \mathcal{F}_p^{-1}(\phi_W) *_p c = \psi \left( \sum_{r=0}^{R} W_r L_p^r c \right) , 
\end{equation}
\noindent where $\phi_W$ is a convolutional filter with learnable parameters $W$, $c \in C^{p}(K)$ is a $p$-cochain on input simplicial complex $K$ (i.e.\ a real valued function over the set of $p$-simplices in $K$) and $\psi$ accounts for the application of bias and non-linearity. In particular, the convolutional filter $\phi_W$ is parameterized as an $R$-degree polynomial of the Hodge $p$-Laplacian $L_p$. By imposing a small degree $R$, it is possible to guarantee spatial filter localization similarly as in graphs.

This proposed convolutional layer can easily be be rewritten in terms of our message passing scheme. Let us first conveniently introduce the following Lemma:
\begin{lemma}
\label{lemma:power}
    The $r$-th power of the Hodge $p$-Laplacian, $L_p^{r}$, is equivalent to the sum of the $r$-th powers of its constituent upper and lower components, that is: $L_p^{r} = (L_p^{\downarrow} + L_p^{\uparrow})^{r} = (L_p^{\downarrow})^{r} + (L_p^{\uparrow})^{r}$.
\end{lemma}
\begin{proof}
    We prove the lemma by induction on the power exponent $r$.
    As the base case, we consider exponent $r=1$, for which the equivalence clearly holds as $L_p^{1} = (L_p^{\downarrow} + L_p^{\uparrow})^{1} = L_p^{\downarrow} + L_p^{\uparrow} = (L_p^{\downarrow})^{1} + (L_p^{\uparrow})^{1}$.
    
    For the induction step, we assume that $L_p^{r-1} = (L_p^{\downarrow})^{r-1} + (L_p^{\uparrow})^{r-1}$ holds true and prove that $L_p^{r} = (L_p^{\downarrow})^{r} + (L_p^{\uparrow})^{r}$. $L_p^{r}$ is defined as:
    $$
        L_p^{r} = L_p L_p^{r-1} = (L_p^{\downarrow} + L_p^{\uparrow}) (L_p^{\downarrow} + L_p^{\uparrow})^{r-1} 
    $$
    
    By the induction hypothesis: 
    \begin{align*}
        (L_p^{\downarrow} &+ L_p^{\uparrow}) (L_p^{\downarrow} + L_p^{\uparrow})^{r-1} = \\
        &= (L_p^{\downarrow} + L_p^{\uparrow}) \bigl((L_p^{\downarrow})^{r-1} + (L_p^{\uparrow})^{r-1}\bigr) = \\
        &= L_p^{\downarrow} (L_p^{\downarrow})^{r-1} + L_p^{\downarrow} (L_p^{\uparrow})^{r-1} \\
        &\quad+ L_p^{\uparrow} (L_p^{\downarrow})^{r-1} + L_p^{\uparrow} (L_p^{\uparrow})^{r-1} = \\
        &= (L_p^{\downarrow})^{r} + L_p^{\uparrow} (L_p^{\downarrow})^{r-1} + L_p^{\downarrow} (L_p^{\uparrow})^{r-1} + (L_p^{\uparrow})^{r}
    \end{align*}
    The second term $L_p^{\uparrow} (L_p^{\downarrow})^{r-1}$ can be rewritten as $B_{p+1} B_{p+1}^\top (B_p^\top B_p)^{r-1} =  B_{p+1} B_{p+1}^\top \underbrace{ B_p^\top B_p }_\text{$(r-1)$ times} = \mathbf{0}$, since $B_{p+1}^\top B_p^\top = (B_p B_{p+1} )^\top$, which equals $\mathbf{0}^\top$ by definition.
    
    Similarly, the third term $ L_p^{\downarrow} (L_p^{\uparrow})^{r-1}$ can be rewritten as $B_p^\top B_p (B_{p+1} B_{p+1}^\top)^{r-1} = B_p^\top B_p \underbrace{ B_{p+1} B_{p+1}^\top }_\text{$(r-1)$ times} = \mathbf{0}$, since, again, $B_p B_{p+1} = \mathbf{0}$.
\end{proof}

We now proceed to prove Theorem~\ref{thm:gen_workshop} that our MPSN can be reduced to SCNN of \citet{ebli2020simplicial} or \citet{bunch2020simplicial}. We split the proofs in two individuals.
\begin{proof}[Proof of Theorem~\ref{thm:gen_workshop} in Reference to \citet{ebli2020simplicial}]
We first rephrase Equation \eqref{eq:SNN} for a generic layer and multi-dimensional input and output $p$-simplicial representations:
\begin{align*}
   H^{t+1} &= \psi 
  \Bigl(\sum_{r=0}^{R} L_p^r H^{t} W^{t+1}_{r}\Bigr)\\ 
  &= \Bigl( H^{t} W^{t+1}_{0} + \sum_{r=1}^{R} L_p^r H^{t} W^{t+1}_{r} \Bigr).
\end{align*}
The convolutional operation for $p$-simplex $\sigma$ can be rewritten as
\begin{align*}
   h_{\sigma}^{t+1} &= \psi \Bigl( h_{\sigma}^{t} W^{t+1}_{0} + \sum_{r=1}^{R} (L_p^r)_\sigma H^{t} W^{t+1}_{r} \Bigr) \\
   &= \psi \Bigl(h_{\sigma}^{t} W^{t+1}_{0} + \sum_{r=1}^{R} \sum_{\tau \in \gS_p} (L_p^r)_{\sigma,\tau} h_{\tau}^{t} W^{t+1}_{r}\Bigr),
\end{align*}
\noindent where $h_{\sigma}^{t+1}$ denotes the $\sigma$-th row of matrix $H^{t+1}$, $(L_p^r)_\sigma$ denotes the $\sigma$-th row of operator $L_p^r$ and $(L_p^r)_{\sigma,\tau}$ its entry at position $\sigma,\tau$. We now leverage on Lemma~\ref{lemma:power} to rewrite the convolution operation on simplex $\sigma$ as
\begin{align*}
   h_{\sigma}^{t+1}
    &= \psi\Bigl( 
        \sum_{r=1}^{R} \sum_{\tau \in \gS_p} \left( ((L_p^{\downarrow})^r)_{\sigma,\tau} + ((L_p^{\uparrow})^r)_{\sigma,\tau}\right) h_{\tau}^{t} W^{t+1}_{r}\\
        &\quad\qquad+ h_{\sigma}^{t} W^{t+1}_{0}
    \Bigr) \\
    &= \psi\Bigl( 
        \sum_{\tau \in \gS_p} \sum_{r=1}^{R} ((L_p^{\downarrow})^r)_{\sigma,\tau} h_{\tau}^{t} W^{t+1}_{r} \\
        &\quad\qquad+ \sum_{\tau \in \gS_p} \sum_{r=1}^{R} ((L_p^{\uparrow})^r)_{\sigma,\tau} h_{\tau}^{t} W^{t+1}_{r} 
        + h_{\sigma}^{t} W^{t+1}_{0}
    \Bigr).
\end{align*}


Considering that matrices $L_p^{\downarrow}$ and $L_p^{\uparrow}$ only convey the notions of, respectively, lower and upper simplex adjacency, the equation above is easily interpreted in terms of our message passing scheme by setting
\begin{align*}
    M_{\uparrow}^{t+1}\big( h_{\sigma}^{t}, h_{\tau}^{t}, h_{\sigma \cup \tau}^{t} \big) &= 
        \sum_{r=1}^{R} ((L_p^{\uparrow})^r)_{\sigma,\tau} h_{\tau}^{t} W^{t+1}_{r}\\
    M_{\downarrow}^{t+1}\big( h_{\sigma}^{t}, h_{\tau}^{t}, h_{\sigma \cap \tau}^{t} \big) &= 
        \sum_{r=1}^{R} ((L_p^{\downarrow})^r)_{\sigma,\tau} h_{\tau}^{t} W^{t+1}_{r}\\
    U^{t+1}\Big( h_{\sigma}^{t}, \{m_{i}^{t}(\sigma)\}_{i=\downarrow,\uparrow} \Big) &= 
        \psi \Bigl( h_{\sigma}^{t} W^{t+1}_{0} \\
        &\quad \qquad+{m_{\uparrow}}_\sigma^{t+1} + {m_{\downarrow}}_\sigma^{t+1} \Bigr),
\end{align*}
\noindent and by letting $\text{AGG}$ be the summation over the extended notion of upper and lower $R$-neighborhoods, that is neighborhoods comprising $p$-simplices at a distance from $\sigma$ which is at most $R$ ($\tau$ is at distance $d$ from $\sigma$ if there exists a sequence of upper- (respectively, lower-) adjacent $p$-simplices $[\nu_0, \nu_1, \mathellipsis, \nu_d]$ such that $\nu_0 = \sigma, \nu_d = \tau$).
\end{proof}
It is noteworthy that, contrary to our general proposed framework, the two message functions $M_{\uparrow}^{t+1}$ and $M_{\downarrow}^{t+1}$ share the same learnable parameters $\{W^{t+1}_{r}\}_{r=1}^{R}$, and that no signal of order lower or higher than $p$ is involved in computation.

\paragraph{SC-Conv} \citet{bunch2020simplicial} proposed a convolutional operator that can be applied on 2-dimensional simplicial complexes. The construction is based on the canonical normalised Hodge Laplacians defined by \citet{schaub2020random}; starting from the operators, the authors generalize the Graph Convolutional Network model proposed in \citet{KiWe2017} by defining the corresponding adjacency matrices with added self-loops:
\begin{align*}
    H_0^{t+1} &= \psi(\mD_1^{-1} \mB_1 H_1^t W^t_{0, 1} + \tilde{\mA}_0^u H_0^t W_{0, 0}^t) \\ 
    H_1^{t+1} &= \psi(\mB_2 \mD_3 H_2^t W^t_{1, 2} \\
    &\qquad+ (\tilde{\mA}_1^d + \tilde{\mA}_1^u) H_1^t W_{1, 1}^t \\ 
    &\quad\qquad+ \mD_2 \mB_1^\top \mD_1^{-1} H_0^t W_{1, 0}^t) \\
    H_2^{t+1} &= \psi(\tilde{\mA}_2^d H_2^t W_{2, 2}^t + \mD_4 \mB_2^\top \mD_5^{-1} H_1^t W_{2, 1}^t).
\end{align*}
We defer readers to Section 2.1 of the original paper for the definitions of the $\{\tilde{A}_{i}^\alpha\}_{i=0}^{2}$ and $\{\mD\}_{i=1}^{5}$ matrices in the above equations. Differently from \citet{ebli2020simplicial}, this scheme models the interactions between signals defined at different dimensions. It can, nonetheless, be rewritten in terms of our message passing framework.
\begin{proof}[Proof of Theorem~\ref{thm:gen_workshop} in Reference to \citet{bunch2020simplicial}]
We report here the derivation for message passing on $1$-simplices (edges) as it is the most general. The derivation on $0$- and $2$-simplices can simply be obtained as a special case of this last.

First, let us conveniently denote:
\begin{align*}
    \Delta_{1,1} = \tilde{\mA}_1^d + \tilde{\mA}_1^u,\;\; 
    \Delta_{1,0} = \mD_2 \mB_1^\top \mD_1^{-1},\;\;
    \Delta_{1,2} = \mB_2 \mD_3.
\end{align*}
We note that the convolutional operation for a generic $1$-simplex $\delta$ can be rewritten as:
\begin{align*}
    h & _{1,\delta}^{t+1} = \\
    &= \psi \Bigl(
        (\Delta_{1,2})_\delta H_{2}^{t} W_{1,2}^{t} + (\Delta_{1,1})_\delta H_{1}^{t} W_{1,1}^{t} \\
        &\enspace+ (\Delta_{1,0})_\delta H_{0}^{t} W_{1,0}^{t}
    \Bigr) \\
    &= \psi \Bigl(
        \sum_{\tau \in \mathcal{S}_2}(\Delta_{1,2})_{\delta,\tau} h_{2,\tau}^{t} W_{1,2}^{t} + \sum_{\gamma \in \mathcal{S}_1}(\tilde{\mA}_1^u)_{\delta,\gamma} h_{1,\gamma}^{t} W_{1,1}^{t} \\
        &\enspace+ \sum_{\gamma \in \mathcal{S}_1}(\tilde{\mA}_1^d)_{\delta,\gamma} h_{1,\gamma}^{t} W_{1,1}^{t} + \sum_{\sigma \in \mathcal{S}_0}(\Delta_{1,0})_{\delta,\sigma} h_{0,\sigma}^{t} W_{1,0}^{t}
    \Bigr) \\
    &= \psi \Bigl(
        \sum_{\tau \in \gC(\delta)}(\Delta_{1,2})_{\delta,\tau} h_{2,\tau}^{t} W_{1,2}^{t} + \bigl( 
            \sum_{\gamma \in \gN_{\uparrow}(\delta)}(\tilde{\mA}_1^u)_{\delta,\gamma} h_{1,\gamma}^{t} \\
            &\enspace+ \sum_{\gamma \in \gN_{\downarrow}(\delta)}(\tilde{\mA}_1^d)_{\delta,\gamma} h_{1,\gamma}^{t} + (\Delta_{1,1})_{\delta,\delta} h_{1,\delta}^{t} \bigr) W_{1,1}^{t} \\
        &\enspace+ \sum_{\sigma \in \gB(\delta)}(\Delta_{1,0})_{\delta,\sigma} h_{0,\sigma}^{t} W_{1,0}^{t}
    \Bigr).
\end{align*}


This equation is interpreted in terms of our message passing scheme by setting:
\begin{align*}
    M_{1,\uparrow}^{t+1}\big( h_{1,\delta}^{t}, h_{1,\gamma}^{t}, h_{2,\delta \cup \gamma}^{t} \big) 
        &= (\tilde{\mA}_1^u)_{\delta,\gamma} h_{1,\gamma}^{t}\\
    M_{1,\downarrow}^{t+1}\big( h_{1,\delta}^{t}, h_{1,\gamma}^{t}, h_{0,\delta \cap \gamma}^{t} \big) 
        &= (\tilde{\mA}_1^d)_{\delta,\gamma} h_{1,\gamma}^{t}\\
    M_{1,\gC}^{t+1}\big( h_{1,\delta}^{t}, h_{2,\tau}^{t} \big)
        &= (\Delta_{1,2})_{\delta,\tau} h_{2,\tau}^{t} \\
    M_{1,\gB}^{t+1}\big( h_{1,\delta}^{t}, h_{0,\sigma}^{t} \big)
        &= (\Delta_{1,0})_{\delta,\sigma} h_{0,\sigma}^{t} \\
    U_{1}^{t+1}\Big( h_{1,\delta}^{t}, \{m_{i}^{t}(\delta)\}_{i=\gB,\gC,\downarrow,\uparrow} \Big) \\
        \quad\quad = \psi \Bigl(
            {W_{1,1}^{t}}^\top \bigl( (\Delta_{1,1})_{\delta,\delta} h_{1,\delta}^{t} 
            &+ {m_{\uparrow}}(\delta)^{t+1} + {m_{\downarrow}}(\delta)^{t+1} \bigr) \\
            &+ {W_{1,2}^{t}}^\top {m_{\gC}}(\delta)^{t+1} \\
            &+ {W_{1,0}^{t}}^\top {m_{\gB}}(\delta)^{t+1}
        \Bigr), \\
\end{align*}
\noindent and $\text{AGG} = \sum$.
\end{proof}

\paragraph{Simplicial Complex Pooling} Typically, CNNs interleave convolutional layers with pooling layers, which spatially downsample the input features. This problem has proven to be much more challenging on graph domains, where no obvious graph coarsening strategy exists. While we have not employed any simplicial complex pooling operators in this work, this is likely going to be an exciting direction of future work. Some of the previously proposed graph pooling operators can readily be extended to simplicial complexes. For instance, TopK pooling \citep{cangea2018towards} can be performed over the vertices of the complex and the higher-order structures associated with the pooled vertices can be maintained at the next layer. Other operators, such as the topologically-motivated pooling proposed by \citet{bodnar2020deep} can already handle simplicial complexes, but it has only been employed in a graph setting. 

\section{Equivariance and Invariance}
\label{app:equiv_inv}

One would expect MPSNs to be aware of the two symmetries of a simplicial complex: relabeling of the simplicies in the complex and, optionally, changes in the orientation of the complex if the complex is oriented. We address these two below.

\paragraph{Permutation Equivariance} Let $\gK$ be simplicial $p$-complex described by a sequence $\rmB = (B_1, \ldots, B_p)$ of boundary matrices  with a sequence $\rmH = (H_0, \ldots, H_p)$ of feature matrices. Let $\rmP$ be a sequence of permutation matrices $(P_0, \ldots, P_p)$ with $P_i \in \sR^{S_i \times S_i}$. Denote by $\rmP \rmH$ the sequence of permuted features $(P_0 H_0, \ldots, P_p H_p)$ and by $\rmP \rmB \rmP^\top$, the sequence of permuted boundary matrices $(P_0 B_1 P_{1}^\top, \ldots, P_{p-1} B_p P_{p}^\top)$. 

\begin{definition}[Permutation equivariance] A function $f: (\rmH^{\rm in}, \rmB) \mapsto \rmH^{\rm out}$ is (simplex) permutation equivariant if $f(\rmP \rmH^{\rm in}, \rmP \rmB \rmP^\top) = \rmP f(\rmH^{\rm in}, \rmB)$ for any sequence of permutation operators $\rmP$. 
\end{definition}

\begin{definition}[Permutation invariance]
Similarly, we say that a function $f$ is (simplex) permutation invariant if $f(\rmP \rmH^{\rm in}, \rmP \rmB \rmP^\top) = f(\rmH^{\rm in}, \rmB)$ for any sequence of permutation operators $\rmP$. 
\end{definition}

\begin{proof}[Proof of Theorem \ref{theo:mpsn_equiv}]
We abuse the notation slightly and use $P_i(a)$ to denote the corresponding permutation function of $P_i$ acting on indices. 

We focus on a single simplex $\sigma$ of an arbitrary dimension $n$ and the corresponding $\tau = P_n(\sigma)$. Let $h_\sigma$ be the output feature of simplex $\sigma$ for an MPSN layer taking $(\rmH, \rmB)$ as input and $\bar{h}_\tau$ the output features of simplex $\tau$ for the same MPSN layer taking $(\rmP \rmH, \rmP \rmB \rmP^\top)$ as input. We will now show they are equal by showing that the multi-set of features being passed to the message, aggregate and update functions are the same for the two simplicies. 

The boundary of the $n$-simplicies in $\rmB$ are given by the non-zero elements of $B_n$. Similarly, in $\rmP \rmB \rmP^\top$, these are given by the non-zero elements of $P_{n-1} B_n P_{n}^\top$. That is the same boundary matrix but with the rows and columns permuted according to $P_{n}$. This then gives
\begin{equation*}
   (B_n)_{a,b} = (P_{n-1} B_n P_{n}^\top)_{P_{n-1}(a), P_n(b)}.
\end{equation*}
In particular, this holds for $b = \sigma, P_n(b) = \tau$.  Because the feature matrices for the $(n-1)$-simplices in $\rmP \rmH$ are also accordingly permuted with $P_{n-1} H_{n-1}$, $\sigma$ and $\tau$ receive the same message from their boundary simplices. The proof follows similarly for co-boundary adjacencies. 

The lower adjacenices of the $n$-simplicies in $\rmB$ are given by the non-zero entries of $B_n^\top B_n$. Similarly, the lower adjacencies in $\rmP \rmB$ are given by the non-zero elements of 
\begin{align*}
   (P_{n-1} B_n P_{n}^\top)^\top (P_{n-1} B_n P_{n}^\top)  
   &= P_{n} B_n^\top P_{n-1}^\top P_{n-1} B_n P_{n}^\top \\
   &= P_{n} B_n^\top B_n P_{n}^\top.  
\end{align*}
That is, the same adjacencies as in $\rmB$, but with the rows and columns are accordingly permuted. Therefore, we obtain
\begin{equation*}
   (B_n^\top B_n)_{a,b} = (P_{n} B_n^\top B_n P_{n}^\top)_{P_n(a), P_n(b)},
\end{equation*}
which, in particular, holds for $a = \sigma,  P_n(a) = \tau$. Since the feature matrices for the $n$-simplices in $\rmP \rmB$ are also permuted with $P_n H_n$, $\sigma$ and $\tau$ receive the same message from the lower adjacenct simplicies. This can be similarly shown for upper adjacencies. 
\end{proof}

Permutation invariance for an MPSN model can be obtained by stacking multiple permutation equivariant layers followed by a permutation invariant readout function. The readout must be permutation invariant in each of the multi-sets of simplices of different dimensions it takes as input. Any commonly used GNN readout functions such as sum or mean could be employed. 

\paragraph{Orientation Equivariance} Another symmetry that we would like to preserve is orientation. For instance, we know that the homology of the complex is invariant to the particular orientation that was chosen. Therefore, for certain applications where orientations are of interest, we would like to design MPSN layers that are orientation equivariant and MPSN networks that are orientation invariant.

When a simplex $\sigma$ changes its orientation, it changes its relative orientation (i.e. $+1$ becomes $-1$ and vice-versa) with respect to all the neighbours in the complex, and the signature of its own features is flipped. Overall, this can be modelled by multiplying the rows and columns of the boundary matrices where $\sigma$ appears and the corresponding row of the feature matrix by $-1$. We formalise this intuition below.   

Consider a simplicial $p$-complex $\gK$. Let $\rmT = (T_0, \ldots, T_p)$ be a sequence of diagonal matrices with $T_i(j, j) = \pm 1$ for all $i > 0$ and any $j$ and $T_0 = I$. The latter constraint is due to the fact that vertices have trivial orientation and it cannot be changed. Using the same notation as for permutation equivariance, define $\rmT \rmH = (T_0 H_0, \ldots, T_p H_p)$ and $\rmT \rmB \rmT = (T_0 B_1 T_1, \ldots, T_{p-1} B_p T_p)$. 

\begin{definition}[Orientation equivariance] A function $f$ mapping $(\rmH^{\rm in}, \rmB) \mapsto \rmH^{\rm out}$ is
orientation equivariant if 
$f(\rmT \rmH^{\rm in}, \rmT \rmB \rmT) = \rmT f(\rmH^{\rm in}, \rmB)$ for any sequence of operators $\rmT$.
\end{definition}

\begin{definition}[Orientation invariance] A function $f$ is orientation invariant if 
$f(\rmT \rmH^{\rm in}, \rmT \rmB \rmT) =  f(\rmH^{\rm in}, \rmB)$ for any sequence of operators $\rmT$.
\end{definition}

\begin{remark}
Orientation invariance can be trivially achieved by considering the absolute value of the features and by treating the complex as an unoriented one. 
\end{remark}

However, it is (in general) desirable to use equivariance at the intermediate layers and make the network invariant with a final transformation (readout). Constructing such an MPSN layer requires imposing additional constraints on the structure of the message, update and aggregate functions. We will start with a simple MPSN layer that can be expressed in matrix form to develop an intuition, and then we will generalise this example. 

For instance, we can consider the model from \eqref{eq:m} used in our linear regions analysis, with a convenient vectorised form
\begin{align}
\label{eq:mpsn_single_level}
    H_i^{\rm{out}} &= \psi \big(B_i^\top B_i H_i^{\rm{in}} W_1 +  H_i^{\rm{in}} W_2 + B_{i+1} B_{i+1}^\top  H_i^{\rm{in}} W_3 \nonumber \\
    &\quad\quad +B_i^\top H^{\rm{in}}_{i-1}W_4 + B_{i+1} H^{\rm{in}}_{i+1}W_5 \big),
\end{align}
where we have separated the upper and lower adjacencies in two and included the feature matrix $H_i$ as an additional term. This corresponds to an MPSN with a message function that multiplies the feature of the neighbour by the relative orientation ($\pm 1$), sum-based aggregation and an update function that adds the incoming messages to its linearly transformed features and passes the output through $\psi$. 

\begin{proposition}
When $\psi$ is an odd activation function, the MPSN layer from Equation \eqref{eq:mpsn_single_level} is orientation equivariant. 
\end{proposition}

\begin{proof}
Let $H_i^{\rm{out}} = f_i(B_{i}, B_{i+1}, H_i^{\text{in}}, H_{i-1}^{\rm{in}}, H_{i+1}^{\text{in}})$ be the application of one such MPSN layer on the $i$-dimensional simplices. Note that the output depends only on the simplices of dimension $i-1, i$, and $i+1$. 

For this MPSN layer to be equivariant, we need to show that 
\begin{align}
\label{eq:orient_equiv_one_level}
f_i&(T_{i-1} B_i T_i, T_i B_{i+1} T_{i+1}, T_i H_i^{\text{in}}, T_{i-1}H_{i-1}^{\text{in}}, T_{i+1} H_{i+1}^{\rm{in}}) \nonumber \\ 
&= T_i H_i^{\rm{out}}
\end{align}
Because $T_i T_i = I$ and $T_i^\top = T_i$ for all $i$, we can easily rewrite LHS as
\begin{align} \label{eq:mpsn_orient_equiv}
    \psi&\Bigl(T_{i}
    \bigl(B_i^\top B_i H_i^{\text{in}} W_1 +  H_i^{\text{in}} W_2 + B_{i+1} B_{i+1}^\top H_i^{\text{in}} W_3 \nonumber \\
    &\quad\quad + B_i^\top H^\text{in}_{i-1}W_4 + B_{i+1} H^\text{in}_{i+1}W_5 \bigr)\Bigr).
\end{align}
Notice that if $\psi$ and $T_i$ commute, then this becomes $T_i H_i^{\text{out}}$. We remark that they commute when $\psi$ is an odd function. \end{proof}

We can generalise this particular architecture to obtain a more general MPSN layer that is orientation equivariant. Based on Equation \eqref{eq:orient_equiv_one_level}, we can see that for an arbitrary simplex $\sigma$, the output features are invariant with respect to changes in the orientation of the neighbours and it is only affected by changes in its own orientation.

Let us rewrite the equivalent constraint from Equation \eqref{eq:orient_equiv_one_level} for a local aggregation of the neighbourhood of a simplex $\sigma$ in the complex. Let us denote by $o_{\sigma, \tau}$, the relative orientation between $\sigma$ and $\tau$. Consider an abstract local neighbourhood aggregation function $A$ taking as input the features $h_\sigma$ and a multi-set of tuples containing the features of the neighbours of $\sigma$ and the relative orientations $\ldblbrace (o_{\sigma, \tau}, h_\tau) \mid \tau \in \gN(\sigma) \rdblbrace$. For brevity, we consider a single generic multi-set of neighbours denoted here by $\gN(\sigma)$, but multiple types can be easily integrated into $A$.  

Based on Equation \eqref{eq:orient_equiv_one_level}, we know $A$ is invariant to changes in the orientations of the neighbours. This means that it is invariant in changes in the signature of the relative orientations and of the features of the neighbours:
\begin{align}
\label{eq:orient_equiv_local_1}
A(h_\sigma, \ldblbrace (o_{\sigma, \tau}, h_\tau) \rdblbrace) = A(h_\sigma, \ldblbrace \pm(o_{\sigma, \tau}, h_\tau) \rdblbrace).
\end{align}
At the same time, we know that the output of $A$ must change its sign when the  orientation of $\sigma$ changes. That is, when all the relative orientations change their signs together with the features of $\sigma$. This leads to a second equation:
\begin{align}
\label{eq:orient_equiv_local_2}
A(h_\sigma, \ldblbrace (o_{\sigma, \tau}, h_\tau) \rdblbrace) = -A(-h_\sigma, \ldblbrace (-o_{\sigma, \tau}, h_\tau) \rdblbrace).
\end{align}
In other words, $A$ must be even in each $(o_{\sigma, \tau}, h_\tau)$ and also odd in $(h_\sigma, \ldblbrace o_{\sigma, \tau} \rdblbrace)$. 

Consider an auxiliary function $A^*$ taking as input the feature vector $h_\sigma$ and a multi-set, such that $A^*$ is an odd function:
$$A^*(h_\sigma, \ldblbrace h_\tau \rdblbrace) = - A^*(-h_\sigma, \ldblbrace -h_\tau \rdblbrace).$$

\begin{lemma}
\label{lemma:local_orient_equiv}
Let $A^*$ be an odd local aggregator as above. Then the function $A$ defined by:
$$A(h_\sigma, \ldblbrace (o_{\sigma, \tau}, h_\tau) \rdblbrace) := A^*(h_\sigma, \ldblbrace o_{\sigma, \tau} \cdot  h_\tau \rdblbrace),$$
where $\cdot$ denotes scalar-vector multiplication, satisfies Equations \ref{eq:orient_equiv_local_1} and \ref{eq:orient_equiv_local_2}. 
\end{lemma}

\begin{proof}
First we prove $A$ satisfies Equation \eqref{eq:orient_equiv_local_1}. Substituting the definition of $A$: 
\begin{align}
    A(h_\sigma, \ldblbrace \pm(o_{\sigma, \tau}, h_\tau) \rdblbrace) &= A^*(h_\sigma, \ldblbrace \pm o_{\sigma, \tau} \cdot \pm h_\tau \rdblbrace) \nonumber \\
    &= A^*(h_\sigma, \ldblbrace o_{\sigma, \tau} \cdot h_\tau \rdblbrace) \nonumber \\
    &= A(h_\sigma, \ldblbrace (o_{\sigma, \tau}, h_\tau) \rdblbrace).
\end{align}
For proving $A$ satisfies Equation \eqref{eq:orient_equiv_local_2} we substitute the definition again: 
\begin{align}
    A(-h_\sigma, \ldblbrace (-o_{\sigma, \tau}, h_\tau) \rdblbrace) &= A^*(-h_\sigma, \ldblbrace  -o_{\sigma, \tau} \cdot h_\tau \rdblbrace) \nonumber \\
    &= -A^*(h_\sigma, \ldblbrace o_{\sigma, \tau} \cdot h_\tau \rdblbrace) \nonumber \\
    &= -A(h_\sigma, \ldblbrace (o_{\sigma, \tau}, h_\tau) \rdblbrace).
\end{align}
\end{proof}

\begin{proof}[\textbf{Proof of Theorem \ref{theo:mpsn_orient_equiv}}]
The proof follows immediately from Lemma \ref{lemma:local_orient_equiv} since the composition of odd functions is odd. Therefore, the function $A^*$ can be implemented using a combination of odd message, aggregate and update functions.  
\end{proof}

It is useful to note that any MLP with odd activation functions and without bias units is an odd function. Additionally, all linear aggregators (mean, sum) are odd. Therefore, most commonly used GNN layers could be adopted by simply using an odd activation function. 

Another important remark is that if the local aggregator ignores the relative orientations, then Equations \ref{eq:orient_equiv_local_1} and \ref{eq:orient_equiv_local_2} simply imply that the local aggregator of the MPSN layer must be odd in the features of $\sigma$ and even in the features of the neighbours. If the features of $\sigma$ are not used in the message function, then an even message function combined with an update function that is odd in $h_\sigma$ also satisfies the equations. 

\begin{remark}
A (permutation invariant) aggregation-based readout layer first applying an element-wise even function $\psi$ to the elements of its input multiset is orientation invariant since
\begin{equation}
    AGG( \ldblbrace \psi(x_i) \rdblbrace) = AGG( \ldblbrace \psi(\pm x_i) \rdblbrace).
\end{equation}
\end{remark}

A flexible way to implement such a layer is to consider a function $\psi(x) = f(x) + f(-x)$, where $f$ is an arbitrary non-odd function that could be parametrised as an MLP. It is easy to see that such a function is even since summation commutes: 
$$\psi(x) = f(x) + f(-x) = f(-x) + f(x) = \psi(-x).$$

To conclude this section, we note that the concurrent work of \citet{glaze2021principled} has also analysed the equivariance properties of simplicial networks in the context of a  specific convolutional operator that is similar to the one in Equation \eqref{eq:mpsn_single_level}. However, our simplicial message passing scheme is more general, and effectively subsumes the convolutional operator presented in their work. Therefore, the equivariance analysis performed here also applies to a much larger family of models. 

\section{Cubical Complexes}

\label{app:cubical}
A message passing approach for cell complexes, a generalisation of simplicial complexes, has  been proposed by \citet{hajij2020cell}. 
 Cell Complexes have a similar hierarchical structure to simplicial complexes, and the approach of \citet{hajij2020cell} is very close to ours. At the same time, arbitrary cell complexes might be too general for certain applications. Here we discuss how our approach could be extended to cubical complexes,  a type of cell complex that from a theoretical point of view is similar to simplicial complexes \cite{kac04}, and which is used in applications. We plan to extend our approach to cubical complexes  in future work.
\begin{figure}[th]
    \centering
     \includegraphics[scale=0.4]{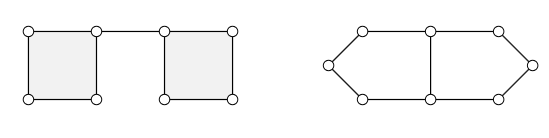}
    \caption{Example of two cubical complexes whose underlying graphs  cannot be distinguished by 1-WL, but are not isomorphic as  cubical complexes.}
    \label{fig:cub_expresiveness}
\end{figure}

\begin{figure}[th]
    \centering
\includegraphics[width=.8\columnwidth]{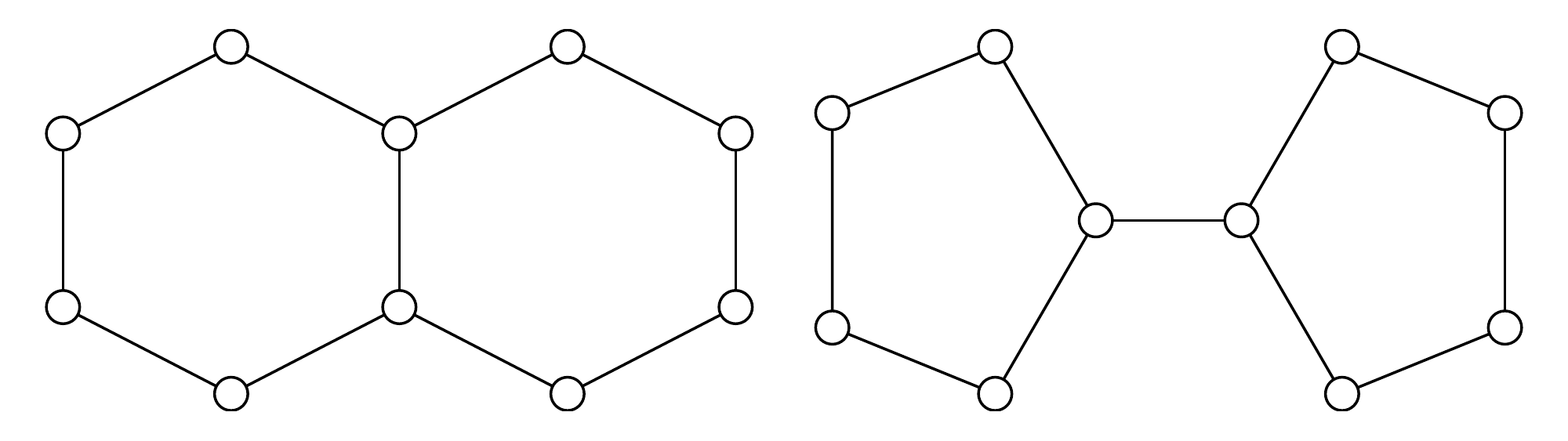}
    \vspace{-5pt}
    \caption{\emph{Decalin} and \emph{Bicyclopentyl}: Non-isomorphic molecular graphs that cannot be distinguished neither by WL nor by SWL, when based on clique complexes (here the nodes represent carbon atoms, and edges are chemical bonds).}
    \label{fig:cyclo}
\end{figure}

Cubical complexes are  cell complexes consisting of unions of points, edges, squares, cubes, and higher-dimensional hypercubes \cite{kac04}. While they  are less well-known than simplicial complexes, they can nevertheless  be used in applications, and  are  better suited than simplicial complexes to the study of certain types of data sets, see, e.g. \citet{WCV11}. Our approach to message passing can be directly implemented also for cubical complexes. One could similarly define a WL test for cubical complexes.

 In Figure \ref{fig:cub_expresiveness} we provide an example of two cubical complexes that are not isomorphic, while their underlying graphs are isomorphic. We note that the clique complexes of the underlying graphs are also isomorphic. 
 The examples in Figures \ref{fig:cc_expresiveness} and \ref{fig:cub_expresiveness} raise the question of whether we could design tests that are better suited to take into account the topological information encoded in simplicial and cubical complexes, such as homeomorphism tests, which have already been studied, see e.g. \citet{BM90}. In particular, such tests should be able to distinguish the graphs in Figure~\ref{fig:cyclo}, a pair of real world molecular graphs that cannot be distinguished by the SWL test based on clique complexes. 

A natural next step in our work is to perform tests on data sets that are more naturally modelled by cubical complexes, such as digital images, for which they can provide computational speed-ups compared to simplicial complexes \cite{WCV11,kac04}.

\section{Additional Experimental Details}\label{app:exps}

\subsection{Strongly Regular Graphs}
The original SR datasets can be found at \url{http://users.cecs.anu.edu.au/˜bdm/data/graphs.html}. Families in Figure~\ref{fig:exp_sr} are referred to as per the SR($v, k, \lambda, \mu$) notation (see Section~\ref{app:higher_order_wl_and_srgs}). We set $\varepsilon=0.01$ as the Euclidean distance threshold for which graphs are deemed non-isomorphic. Graphs are embedded in a $16$-dimensional space by running an untrained, $5$-layer, MPSN model on the associated clique complexes, obtained by considering any $(p+1)$-clique as a $p$-simplex in the complex. Nodes are initialised with a constant, scalar, unitary signal, while higher-order simplices with the sum of the features of the constituent nodes. As already mentioned in the main text, the specific MPSN architecture is dubbed `SIN' given its resemblance to the GIN model \citep{GIN}. The following message passing operations are employed to compute the $t+1$ intermediate representation for a $p$-simplex $\sigma$:
\begin{align}\label{eq:sin_mp_equations_sr}
    h_{\sigma}^{t+1} &= \mathrm{MLP}^{t}_{U,p} \Big( 
        \mathrm{MLP}^{t}_{\gB,p} \big( (1+\eps_{\gB}) h_{\sigma}^{t} + \sum_{\delta \in \gB(\sigma)} h_{\delta}^{t} \big) \parallel \nonumber \\
    &\enspace \mathrm{MLP}^{t}_{\uparrow,p} \big( (1+\eps_{\uparrow}) h_{\sigma}^{t} + \sum_{\tau \in \gN_{\uparrow}(\sigma)} M_{\uparrow,p}^{t} ( h_{\tau}^{t}, h_{\sigma \cup \tau}^t ) \big) \Big) \nonumber \\
    M_{\uparrow,p}^{t} &( h_{\tau}^{t}, h_{\sigma \cup \tau}^t ) = \mathrm{MLP}^{t}_{M,p} \big( h_{\tau}^{t} \parallel h_{\sigma \cup \tau}^t \big),
\end{align}
\noindent where $\parallel$ indicates concatenation, $\mathrm{MLP}^{t}_{\gB,p}$ and $\mathrm{MLP}^{t}_{\uparrow,p}$ are 2-Layer Perceptrons and $\mathrm{MLP}^{t}_{U,p}$, $\mathrm{MLP}^{t}_{M,p}$ consist of a dense layer followed by a non-linearity. Parameters $\eps_{\gB}$, $\eps_{\uparrow}$ are set to zero. Notice how, in accordance with Theorem~\ref{thm:sparse swl}, messages are only aggregated from boundary and upper-adjacent simplices (for which we also include the representations of the shared coboundary simplices).
After the $L=5$ message passing layers, for a $d$-complex $\gK$ we perform readout operations as follows:
\begin{align}
    h_\gK = \mathrm{MLP}^{\gK}\Big( \sum_{p=0}^d \mathrm{MLP}^{\gS}_{p}\big( \sum_{\sigma \in \gS_p} h^{L}_\sigma \big) \Big),
\end{align}
\noindent where $\gS_p$ is the set of simplices at dimension $p$, and $\mathrm{MLP}^{\gK}$, $\mathrm{MLP}^{\gS}_{p}$ are dense layers followed by non-linearities, set to ELU \citep{ELU} throughout the whole architecture. All layer sizes are set to $16$, except for the one of $\mathrm{MLP}^{\gS}_{p}$ layers, which is set to $2 \cdot 16 = 32$.

The `MLP-sum' baseline replaces the $5$ message passing layers by applying a independent dense layer, followed by non-linearity, at each complex dimension. The readout scheme replicates the one in SIN we just described above. In this model, the initial dense layers have size $256$, $\mathrm{MLP}^{\gS}_{p}$ layers have size $2 \cdot 256= 512$ and $\mathrm{MLP}^{\gK}$ has size $16$ as in SIN. We apply ELU non-linearities in this model as well.

Finally, we report here that we ran a comparable GIN architecture and empirically verified its inability to distinguish any pair, theoretically justified by its expressive power being upper-bounded by $1$-WL~\citep{GIN}.

\begin{table*}[!t]
    \centering
    \caption{Hyperparameter configurations on TUDatasets.}
    \label{tab:tu_hyper}
    \vspace{1mm}
    \resizebox{0.6\textwidth}{!}{
    \begin{tabular}{l|cccccc}
        \toprule
        Hyperparameter &
            PROTEINS &
            NCI1 &
            IMDB-B &
            IMDB-M &
            RDT-B &
            RDT-M5K \\
        \midrule
        Batch Size &
            128 &
            128 &
            32 &
            128 &
            32 &
            32 \\
        Initial LR &
            0.01 &
            0.001 &
            0.0005 &
            0.003 &
            0.001 &
            0.001 \\
        LR Dec. Steps &
            20 &
            50 &
            20 &
            50 &
            50 &
            20 \\
        LR Dec. Strength &
            0.5 &
            0.5 &
            0.5 &
            0.5 &
            0.5 &
            0.9 \\
        Hidden Dim. &
            32 &
            32 &
            64 &
            64 &
            64 &
            64 \\
        Drop. Rate &
            0.5 &
            0.5 &
            0.5 &
            0.5 &
            0.0 &
            0.0 \\
        Drop. Pos. &
            a &
            a &
            a &
            a &
            b &
            b \\
        Num. Layers &
            3 &
            4 &
            4 &
            4 &
            4 &
            4 \\
        \bottomrule
    \end{tabular}
    }
\end{table*}

\subsection{Trajectory Prediction}

\paragraph{Dataset Details} To generate the synthetic complex, we uniformly sample 1,000 points in the unit square, we perform a Delaunay triangulation of these points and then remove the triangles (and points) intersecting with two pre-defined regions of the plane to create the two holes. To generate the trajectories, we first randomly sample a random point from the top-left corner of the complex and an end point from the bottom-right corner of the complex. We then perform a random walk on the edges of the complex. With a probability of $0.9$, the neighbour closest to the end-point is chosen, and with a probability $0.1$, a random neighbour is chosen. To generate the two classes, we set random points either from the bottom-left corner or the top-right corner as an intermediate checkpoint. We generate 1,000 train trajectories and 200 test trajectories.

The ocean drifter benchmark has been designed in light of a study conducted by~\citet{schaub2020random} showing that clock- and counterclockwise drifter trajectories around Madagascar can be disentangled when projected on the harmonic subspace of a normalised Hodge $1$-Laplacian. Such operator is the one associated with the simplicial complex representing the underlying geographical map. The original drifter measurements have been collected by the Global Drifter Program, and have been made publicly available by the U.S. National Oceanic and Atmospheric Administration (NOAA) at \url{http://www.aoml.noaa.gov/envids/gld/}. The simplicial complex structure around the Madagascar island is constructed in accordance with~\citet{schaub2020random} and~\citet{glaze2021principled}, i.e. by considering an underlying hexagonal tiling and by generating $2$-simplices from the triangles induced by the local adjacency between tiles. Edge-flows are also determined consistently with these works, that is by drawing trajectories based on the position of measurement buoys relative to the underlying hexagonal tiles.
We readapted the preprocessing script from~\citet{glaze2021principled}, which is publicly available at \url{https://github.com/nglaze00/SCoNe_GCN/tree/master/ocean_drifters_data}. The final dataset we use consists of 160 train trajectories and 40 test trajectories. 

For both benchmarks, the simplicial complexes in the training dataset use a fixed arbitrary orientation, while each test trajectory uses a random orientation obtained by multiplying the boundary and feature matrices with a random diagonal matrix $T_1$ with $\pm1$ entries as in Appendix \ref{app:equiv_inv}. 

\paragraph{Architecture} All the evaluated models use a similar architecture. The MPSN Id, Tanh and ReLU models use layers of the form:
\begin{align}
    \psi \Big( W_0 h^t_\sigma +
        \sum_{\tau \in \ndown(\sigma)} W_1 h_{\tau}^{t} o_{\sigma, \tau} + \sum_{\delta \in \gN_{\uparrow}(\sigma)} W_2 h_{\delta}^{t} o_{\sigma, \delta}
    \Big), \nonumber
\end{align}
where $\psi$ represents the non-linearity from the model's name and $o_{\sigma, \tau} = \pm 1$ represents the relative orientation between $\sigma$ and $\tau$. Based on Theorem \ref{theo:mpsn_orient_equiv}, this layer is orientation equivariant if $\psi$ is an odd function. Thus, MPSN Id and MPSN Tanh are equivariant, whereas MPSN ReLU is not. 

The $L_0$-inv MPSN uses a similar layer, but drops the relative orientations: 
\begin{equation*}
    \psi \Big( W_0 h^t_\sigma +
        \sum_{\tau \in \ndown(\sigma)} W_1 h_{\tau}^{t} + \sum_{\delta \in \gN_{\uparrow}(\sigma)} W_2 h_{\delta}^{t}
    \Big). 
\end{equation*}
The GNN $L_0$-inv model uses the same layer, but without upper adjacencies:
\begin{equation*}
    \psi \Big( W_0 h^t_\sigma +
        \sum_{\tau \in \ndown(\sigma)} W_1 h_{\tau}^{t} 
    \Big).
\end{equation*}
Additionally, the $L_0$-inv models are made orientation invariant by taking the absolute value of the input features before passing them through the network. 

All models use a sum-readout followed by an MLP. The MPSN Id, Tanh and ReLU models use the absolute value of the features before the readout. 

\paragraph{Hyperparameters \& Evaluation Procedure} All the models use the same hyperparameters. We use $4$ layers, the hidden size is set to $64$, the batch size is $64$, and the initial learning rate is set to $0.001$ and decayed by a factor of $0.5$ with a dataset-depending frequency. On the synthetic dataset we train for $100$ epochs and reduce the learning rate every $20$ epochs. On the ocean drifter dataset we train for $250$ epochs and reduce the learning rate every $50$ epochs. We report the final training and test accuracy at the end of training for all models over five seeds.

\subsection{Real-World Graph Classification}

The graph classification tasks from this set of experiments are commonly used to benchmark GNNs and are available at \url{https://chrsmrrs.github.io/datasets/}. We lift each graph to a clique $2$-complex, where $0$-simplices are initialised with the original node signals as prescribed in~\citet{GIN}. Higher dimensional simplices are initialised with the mean of the constituent nodes. We employ a SIN model which applies the following message passing scheme to compute the $t+1$ intermediate representation for $p$-simplex $\sigma$:
\begin{align}\label{eq:sin_mp_equations_tu}
    h_{\sigma}^{t+1} &= 
        \mathrm{MLP}^{t}_{U,p} \Big( 
            \mathrm{MLP}^{t}_{\gB,p}\big( (1+\eps_{\gB}) h_{\sigma}^{t} + \sum_{\delta \in \gB(\sigma)} h_{\delta}^{t} \big) \parallel \nonumber \\
    &\qquad \mathrm{MLP}^{t}_{\uparrow,p}\big( (1+\eps_{\uparrow}) h_{\sigma}^{t} + \sum_{\tau \in \gN_{\uparrow}(\sigma)} h_{\tau}^{t} \big)\Big),
\end{align}
\noindent where $\parallel$ indicates concatenation, $\mathrm{MLP}^{t}_{\gB, p}$ and $\mathrm{MLP}^{t}_{\uparrow, p}$ are 2-Layers Perceptrons endowed with Batch Normalization \citep{BN} (BN) and ReLU activations, $\mathrm{MLP}^{t}_{U,p}$ is a dense layer followed by the application of BN and ReLU. The only exception is represented by Reddit datasets, on which BN was observed to cause instabilities in the training process and was not applied. Parameters $\eps_{\gB}$ and $\eps_{\uparrow}$ are set to zero and are not optimised. As it is possible to notice in Equation~(\ref{eq:sin_mp_equations_tu}), upper message $m_{\uparrow, p}(\sigma)$ is computed as $m_{\uparrow, p}(\sigma) = \sum_{\tau \in \gN_{\uparrow}(\sigma)} h_{\tau}^{t}$, thus explicitly disregarding the representation of shared co-boundary simplices $h_{\sigma \cup \tau}^t$: this choice showed to yield better performance on these benchmarks. We follow~\citet{GIN} and apply a Jumping Knowledge (JK) scheme \citep{JK}: at dimension $p$, readout is performed on the concatenation of the $p$-simplex representations obtained at each message passing iteration, projected by the non-linear dense layer $\mathrm{MLP}^{\gS}_{p}$. Such a readout operation is dataset specific: we apply either averaging or summation as prescribed by~\citet{GIN}. Final complex representations are obtained by summing the obtained $p$-embeddings at dimensions $0$ (`nodes') and $2$ (`triangles'). One last dense layer is applied to output class predictions. Training is performed with Adam optimiser~\cite{kingma2014adam}, starting from an initial learning rate decayed with a certain frequency.

We employ grid-search to tune batch size, hidden dimension, dropout rate, initial learning rate along with its decay steps and strengths, number of layers and the dropout position (before the final readout on the complex (`b'), or after it (`a'). We report the hyperparameter configurations in Table~\ref{tab:tu_hyper}. As in experiments on SR graphs, the size of $\mathrm{MLP}^{\gS}_{p}$ layers is doubled w.r.t.\ the other layers in the architecture.

\subsection{Implementation \& Availability}

We employed PyTorch~\citep{NEURIPS2019_9015} for all our experiments and we built on top of the PyTorch Geometric library~\citep{fey2019fast} to implement our simplicial message passing scheme. 

As for clique-lifting procedures, we relied on the simplex tree data-structure \citep{Boissonnat2014TheST} implemented in the topological data analysis library Gudhi \citep{gudhi:urm}. Empirically, in large-scale experiments involving graphs with $10^6$ nodes, simplex trees are able to generate the clique complex up to a constant desired dimension in a computational cost that is linear in the number of simplicies in the complex \citep{Boissonnat2014TheST}.

We refer readers to \url{https://arxiv.org/abs/2103.03212} 
for a link to our official code repository. We provide support for SC datasets, simplicial message passing networks, oriented SCs, (higher-order) batching, clique complexes and other additional features.

\end{document}

%% file: math_commands.tex
\usepackage{amsmath,amsfonts,bm,amsthm,amssymb}

\newtheorem{theorem}{Theorem}
\newtheorem{proposition}[theorem]{Proposition}
\newtheorem{definition}[theorem]{Definition}

\def\eqref#1{equation~\ref{#1}}

\def\1{\bm{1}}

\def\eps{{\epsilon}}

\def\rmB{{\mathbf{B}}}

\def\rmH{{\mathbf{H}}}

\def\rmP{{\mathbf{P}}}

\def\rmT{{\mathbf{T}}}

\def\mA{A}
\def\mB{B}

\def\mD{D}

\DeclareMathAlphabet{\mathsfit}{\encodingdefault}{\sfdefault}{m}{sl}
\SetMathAlphabet{\mathsfit}{bold}{\encodingdefault}{\sfdefault}{bx}{n}

\def\gB{{\mathcal{B}}}
\def\gC{{\mathcal{C}}}

\def\gK{{\mathcal{K}}}

\def\gN{{\mathcal{N}}}
\def\gO{{\mathcal{O}}}

\def\gS{{\mathcal{S}}}

\def\gV{{\mathcal{V}}}

\def\sC{{\mathbb{C}}}

\def\sR{{\mathbb{R}}}

\newcommand{\R}{\mathbb{R}}

\newcommand{\nup}{\gN_{\uparrow}}
\newcommand{\ndown}{\gN_{\downarrow}}

\newcommand{\da}{\downarrow}
\newcommand{\ua}{\uparrow}
